%% file: ms.tex
\documentclass{article}

\usepackage{times}
\usepackage{graphicx} 

\usepackage{natbib}

\usepackage{algorithm}
\usepackage[noend]{algorithmic}
\renewcommand{\algorithmicrequire}{\textbf{Input:}}
\renewcommand{\algorithmicensure}{\textbf{Output:}}

\usepackage{hyperref}

\usepackage[accepted]{icml2017}

\icmltitlerunning{Lost Relatives of the Gumbel Trick}

\usepackage{amsmath}		
\usepackage{amsfonts}		
\usepackage{bbm}            
\usepackage{booktabs}		
\usepackage{dsfont}
\usepackage{mathtools}		
\usepackage{subcaption}		

\usepackage{amsthm}
\theoremstyle{plain}
\newtheorem{theorem}{Theorem}
\newtheorem{corollary}[theorem]{Corollary}
\newtheorem{lemma}[theorem]{Lemma}
\newtheorem{proposition}[theorem]{Proposition}
\theoremstyle{remark}
\newtheorem*{remark}{Remark}
\theoremstyle{definition}
\newtheorem{definition}[theorem]{Definition}

\newtheorem{example}[theorem]{Example}

\usepackage{thmtools}
\declaretheoremstyle[notefont=\bfseries,notebraces={}{},%
    headpunct={},postheadspace=1em]{mystyle}
\declaretheorem[style=mystyle,numbered=no,name=Lemma]{hlem} 
\declaretheorem[style=mystyle,numbered=no,name=Proposition]{hprop} 
\declaretheorem[style=mystyle,numbered=no,name=Corollary]{hcorl} 

\usepackage{enumitem}
\setlist{nosep}

\DeclarePairedDelimiterX{\infdivx}[2]{(}{)}{%
	#1\;\delimsize\|\;#2%
}
\newcommand{\KL}{\operatorname{KL}\infdivx}

\newcommand{\argmin}{\operatorname*{argmin}}
\newcommand{\argmax}{\operatorname*{argmax}}
\newcommand{\pow}{\operatorname*{pow}}

\renewcommand{\d}{\,\mathrm{d}}			
\newcommand{\bx}{\mathbf{x}}			
\newcommand{\calL}{\mathcal{L}}			
\newcommand{\calU}{\mathcal{U}}			
\newcommand{\calX}{\mathcal{X}}			
\newcommand{\IE}{\mathbb{E}}  			
\newcommand{\IP}{\mathbb{P}}  			
\newcommand{\IR}{\mathbb{R}}  			

\newcommand{\bias}{\operatorname{bias}}
\newcommand{\var}{\operatorname{var}}
\newcommand{\MSE}{\operatorname{MSE}}

\newcommand{\Exp}{\operatorname{Exp}}

\newcommand{\InvGamma}{\operatorname{InvGamma}}
\newcommand{\Gumbel}{\operatorname{Gumbel}}
\newcommand{\Frechet}{\operatorname{Fr\acute{e}chet}}
\newcommand{\Weibull}{\operatorname{Weibull}}

\newcommand{\qsum}{q_{\text{sum}}}		
\newcommand{\qavg}{q_{\text{avg}}}		

\begin{document}

\twocolumn[
\icmltitle{Lost Relatives of the Gumbel Trick}



\icmlsetsymbol{equal}{*}

\begin{icmlauthorlist}
\icmlauthor{Matej Balog}{cam,tueb}
\icmlauthor{Nilesh Tripuraneni}{berk}
\icmlauthor{Zoubin Ghahramani}{cam,uber}
\icmlauthor{Adrian Weller}{cam,ati}
\end{icmlauthorlist}

\icmlaffiliation{cam}{University of Cambridge, UK}
\icmlaffiliation{tueb}{MPI-IS, T\"{u}bingen, Germany}
\icmlaffiliation{berk}{UC Berkeley, USA}
\icmlaffiliation{uber}{Uber AI Labs, USA}
\icmlaffiliation{ati}{Alan Turing Institute, UK}

\icmlcorrespondingauthor{Matej Balog}{first.last@gmail.com}

\icmlkeywords{machine learning, statistics, partition function, Gumbel trick, ICML}

\vskip 0.3in
]



\printAffiliationsAndNotice{}  

\begin{abstract}
The Gumbel trick is a method to sample from a discrete probability distribution, or to estimate its normalizing partition function.
The method relies on repeatedly applying a random perturbation to the distribution in a particular way, each time solving for the most likely configuration.
We derive an entire family of related methods, of which the Gumbel trick is one member, and show that the new methods have superior properties in several settings with minimal additional computational cost.
In particular, for the Gumbel trick to yield computational benefits for discrete graphical models, Gumbel perturbations on all configurations are typically replaced with so-called low-rank perturbations.
We show how a subfamily of our new methods adapts to this setting, proving new upper and lower bounds on the log partition function and deriving a family of sequential samplers for the Gibbs distribution.
Finally, we balance the discussion by showing how the simpler analytical form of the Gumbel trick enables additional theoretical results.
\end{abstract}

\section{Introduction}
\label{sec:Introduction}


In this work we are concerned with the fundamental problem of sampling from a discrete probability distribution and evaluating its normalizing constant.
A probability distribution $p$ on a discrete sample space $\calX$ is provided in terms of its potential function $\phi: \calX \to [-\infty, \infty)$, corresponding to log-unnormalized probabilities via $p(\bx) = e^{\phi(\bx)} / Z$, where the normalizing constant $Z$ is the \emph{partition function}.
In this context, $p$ is the \emph{Gibbs distribution} on $\calX$ associated with the potential function $\phi$. The challenges of sampling from such a discrete probability distribution and estimating the partition function are fundamental problems with ubiquitous applications in machine learning, classical statistics and statistical physics (see, e.g.,~\citealp{opac-b1079282}).


\emph{Perturb-and-MAP} methods~\citep{papandreou_gaussian_2010} constitute a class of randomized algorithms for estimating partition functions and sampling from Gibbs distributions, which operate by randomly perturbing the corresponding potential functions and employing maximum a posteriori (MAP) solvers on the perturbed models to find a maximum probability configuration.
This MAP problem is NP-hard in general; however, substantial research effort has led to the development of solvers which can efficiently compute or estimate the MAP solution on many problems that occur in practice (e.g., \citealp{boykov2001fast,kolmogorov2006convergent,Dar09}).
Evaluating the partition function is a harder problem, containing for instance \#P-hard counting problems.
The general aim of perturb-and-MAP methods is to reduce the problem of partition function evaluation, or the problem of sampling from the Gibbs distribution, to repeated instances of the MAP problem (where each instance is on a different random perturbation of the original model).

The Gumbel trick~\citep{papandreou_perturb_map_2011} relies on adding Gumbel-distributed noise to each configuration's potential $\phi(\bx)$.
We derive a wider family of perturb-and-MAP methods that can be seen as perturbing the model in different
ways --  in particular
using the Weibull and Fr\'echet distributions alongside the Gumbel. We show that the new methods
can be implemented with essentially no additional computational cost by simply averaging existing Gumbel MAP perturbations in different spaces, and that they can lead to more accurate estimators of the partition function.

Evaluating or perturbing each configuration's potential with i.i.d.~Gumbel noise can be computationally expensive. One way to mitigate this is to cleverly prune computation in regions where the maximum perturbed potential is unlikely to be found \cite{maddison_ast_2014,chen_scalable_2016}.
Another approach exploits the product structure of the sample space in discrete graphical models, replacing i.i.d.~Gumbel noise with a ``low-rank" approximation. \citet{hazan_partition_2012, hazan_sampling_2013} showed that from such an approximation, upper and lower bounds on the partition function and a sequential sampler for the Gibbs distribution can still be recovered. We show that a subfamily of our new methods, consisting of Fr\'echet, Exponential and Weibull tricks, can also be used with low-rank perturbations, and use these tricks to derive new upper and lower bounds on the partition function, and to construct new sequential samplers for the Gibbs distribution.

Our main contributions are as follows:
\begin{enumerate}[leftmargin=*]
	\item A family of tricks that can be implemented by simply averaging Gumbel perturbations in different spaces, and which can lead to more accurate or more sample efficient estimators of $Z$ (Section~\ref{sec:NewTricks}).
	\item New upper and lower bounds on the partition function of a discrete graphical model computable using low-rank perturbations, and a corresponding family of sequential samplers for the Gibbs distribution (Section~\ref{sec:LowRank}).


	\item Discussion of advantages of the simpler analytical form of the Gumbel trick including new links between the errors of
	estimating $Z$,
	sampling, and entropy estimation using low-rank Gumbel perturbations (Section~\ref{sec:LowRankErrors}).
\end{enumerate}


\paragraph{Background and Related work}
\label{sec:RelatedWork}

The idea of perturbing the potential function of a discrete graphical model in order to sample from its associated Gibbs distribution was introduced by~\citet{papandreou_perturb_map_2011}, inspired by their previous work on reducing the sampling problem for Gaussian Markov random fields to the problem of finding the mean, using independent local perturbations of each Gaussian factor~\citep{papandreou_gaussian_2010}. \citet{tarlow_randomized_2012} extended this perturb-and-MAP approach to sampling, in particular by considering more general structured prediction problems. \citet{hazan_partition_2012} pointed out that MAP perturbations are useful not only for sampling the Gibbs distribution (considering the argmax of the perturbed model), but also for bounding and approximating the partition function (by considering the value of the max).

Afterwards,~\citet{hazan_sampling_2013} derived new lower bounds on the partition function and proposed a new sampler for the Gibbs distribution that samples variables of a discrete graphical model sequentially, using expected values of low-rank MAP perturbations to construct the conditional probabilities. Due to the low-rank approximation, this algorithm has the option to reject a sample. \citet{orabona_measure_2014} and~\citet{hazan_high_2016} subsequently derived measure concentration results for the Gumbel distribution that can be used to control the rejection probability. \citet{maji_active_2014} derived an uncertainty measure from random MAP perturbations, using it within a Bayesian active learning framework for interactive image boundary annotation.

Perturb-and-MAP was famously generalized to continuous spaces by~\citet{maddison_ast_2014}, replacing the Gumbel distribution with a Gumbel process and calling the resulting algorithm \emph{A* sampling}. \citet{maddison2016poisson} cast this work into a unified framework together with adaptive rejection sampling techniques, based on the notion of exponential races. This recent view generally brings together perturb-and-MAP and accept-reject samplers, exploiting the connection between the Gumbel distribution and competing exponential clocks that we also discuss in Section~\ref{sec:GumbelTrick}.

Inspired by A* sampling,~\citet{kim_2016_linear_programming} proposed an exact sampler for discrete graphical models based on lazily-instantiated random perturbations, which uses linear programming relaxations to prune the optimization space. Further recent applications of perturb-and-MAP include structured prediction in computer vision~\citep{bertasius_local_2016} and turning the discrete sampling problem into an optimization task that can be cast as a multi-armed bandit problem~\citep{chen_scalable_2016}, see Section~\ref{sec:exp:FullRank} below.

In addition to perturb-and-MAP methods, we are aware of three other approaches to estimate the partition function of a discrete graphical model via MAP solver calls.  The WISH method (weighted-integrals-and-sums-by-hashing, \citealp{Erm13icml}) relies on repeated MAP inference calls applied to the model after subjecting it to random hash constraints. The Frank-Wolfe method may be applied by iteratively updating marginals using a constrained MAP solver and line search \citep{BelSheMcC13, krishnan_barrier_2015}. \citet{WelJeb14} instead use just one MAP call over a discretized mesh of marginals to approximate the Bethe partition function, which itself is an estimate (which often performs well) of the true partition function.

\section{Relatives of the Gumbel Trick}
\label{sec:NewTricks}

In this section, we review the Gumbel trick and state the mechanism by which it can be generalized into an entire family of tricks. We show how these tricks can equivalently be viewed as averaging standard Gumbel perturbations in different spaces, instantiate several examples, and compare the various tricks' properties.

\paragraph{Notation} Throughout this paper, let $\calX$ be a finite sample space of size $N := |\calX|$. Let $\tilde{p} : \calX \to [0, \infty)$ be an unnormalized mass function over $\calX$ and let $Z := \sum_{x \in \calX} \tilde{p}(x)$ be its normalizing partition function. Write $p(x) := \tilde{p}(x) / Z$ for the normalized version of $\tilde{p}$, and $\phi(x) := \ln \tilde{p}(x)$ for the log-unnormalized probabilities, i.e. the potential function.

We write $\Exp(\lambda)$ for the exponential distribution with rate (inverse mean) $\lambda$ and $\Gumbel(\mu)$ for the Gumbel distribution with location $\mu$ and scale $1$. The latter has mean $\mu + c$, where $c \approx 0.5772$ is the Euler-Mascheroni constant.

\begin{table*}[!tb]
	\caption{New tricks for constructing unbiased estimators of different transformations $f(Z)$ of the partition function.}
	\vspace{-0.5em}
	\begin{center}
	\begin{small}
		\begin{tabular}{lllll}
			\toprule
			Trick &
			$g(x)$ &
			Mean $f(Z)$ &
			Variance of $g(T)$ &
			Asymptotic var. of $\hat{Z}$
			\\ \midrule
			Gumbel &
			$-\ln x - c$ &
			$\ln Z$ &
			$\frac{\pi^2}{6}$ &
			$\frac{\pi^2}{6} Z^2$
			\\
			Exponential &
			$x$ &
			$\frac{1}{Z}$ &
			$\frac{1}{Z^2}$ &
			$Z^2$
			\\
			Weibull $\alpha$ &
			$x^{\alpha}$, $\alpha > 0$ &
			$Z^{-\alpha} \Gamma(1 + \alpha)$ &
			$\frac{\Gamma(1 + 2\alpha) - \Gamma(1 + \alpha)^2}{Z^{2\alpha}}$ &
			$\frac{1}{\alpha^2} \left( \frac{\Gamma(1 + 2\alpha)}{\Gamma(1+\alpha)^2} - 1 \right) Z^2$
			\\
			Fr\'echet $\alpha$ &
			$x^{\alpha}$, $\alpha \in (-1, 0)$ &
			$Z^{-\alpha} \Gamma(1 + \alpha)$ &
			$\frac{\Gamma(1 + 2\alpha) - \Gamma(1 + \alpha)^2}{Z^{2\alpha}}$ for $\alpha > - \frac{1}{2}$ &
			$\frac{1}{\alpha^2} \left( \frac{\Gamma(1 + 2\alpha)}{\Gamma(1+\alpha)^2} - 1 \right) Z^2$
			\\
			Pareto &
			$e^x$ &
			$\frac{Z}{Z - 1}$ for $Z > 1$ &
			$a \frac{Z}{(Z-1)^2(Z-2)}$ for $Z > 2$ &
			$\frac{Z^2}{(Z - 2)^2}$
			\\
			Tail $t$ &
			$\mathds{1}_{\{x > t\}}$ &
			$e^{-tZ}$ &
			$e^{-tZ} (1 - e^{-tZ})$ &
			$\frac{(1 - e^{-tZ})^2}{t^2}$
			\\ \bottomrule
		\end{tabular}
	\end{small}
	\end{center}
	\vspace{-0.5em}
	\label{tab:NewTricks}
\end{table*}

\subsection{The Gumbel Trick}
\label{sec:GumbelTrick}

Similarly to the connection between the Gumbel trick and the Poisson process established by \citet{maddison2016poisson}, we introduce the Gumbel trick for discrete probability distributions using a simple and elegant construction via \emph{competing exponential clocks}.
Consider $N$ independent clocks, started simultaneously, such that the $j$-th clock rings after a random time $T_j \sim \Exp(\lambda_j)$. Then it is easy to show that (1) the time until some clock rings has $\Exp(\sum_{j=1}^N \lambda_j)$ distribution, and (2) the probability of the $j$-th clock ringing first is proportional to its rate $\lambda_j$.
These properties are also widely used in survival analysis \citep{cox1984analysis}.

Consider $N$ competing exponential clocks $\{ T_x \}_{x \in \calX}$, indexed by elements of $\calX$, with respective rates $\lambda_x = \tilde{p}(x)$. Property (1) of competing exponential clocks tells us that
\begin{equation}
\min_{x \in \calX} \{ T_x \} \sim \Exp(Z).
\label{eq:ExpClockZ}
\end{equation}
Property (2) says that the random variable $\argmin_x T_x$, taking values in $\calX$, is distributed according to $p$:
\begin{equation}
\argmin_{x \in \calX} \{ T_x \} \sim p.
\label{eq:ExpClockSample}
\end{equation}
The Gumbel trick is obtained by applying the function $g(x) = - \ln x - c$ to the equalities in distribution (\ref{eq:ExpClockZ}) and (\ref{eq:ExpClockSample}). When $g$ is applied to an $\Exp(\lambda)$ random variable, the result follows the $\Gumbel(-c + \ln \lambda)$ distribution, which can also be represented as $\ln \lambda + \gamma$, where $\gamma \sim \Gumbel(-c)$.
Defining $\{ \gamma(x) \}_{x \in \calX} \stackrel{\text{i.i.d.}}{\sim} \Gumbel(-c)$ and noting that $g$ is strictly decreasing, applying the function $g$ to equalities in distribution (\ref{eq:ExpClockZ}) and (\ref{eq:ExpClockSample}), we obtain:
\begin{align}
\max_{x \in \calX} \{ \phi(x) + \gamma(x) \}
&\sim
\Gumbel(-c + \ln Z),
\tag{\ref{eq:ExpClockZ}'}
\label{eq:GumbelZ}
\\
\argmax_{x \in \calX} \{ \phi(x) + \gamma(x) \}
&\sim
p,
\tag{\ref{eq:ExpClockSample}'}
\label{eq:GumbelSample}
\end{align}
where we have recalled that $\phi(x) = \ln \lambda_x = \ln \tilde{p}(x)$.
The distribution $\Gumbel(-c + \ln Z)$ has mean $\ln Z$, and thus the log partition function can be estimated by averaging samples~\citep{hazan_partition_2012}.

\subsection{Constructing New Tricks}

Given the equality in distribution (\ref{eq:ExpClockZ}), we can treat the problem of estimating the partition function $Z$ as a parameter estimation problem for the exponential distribution. Applying the function $g(x) = - \ln x - c$ as in the Gumbel trick to obtain a $\Gumbel(-c + \ln Z)$ random variable, and estimating its mean to obtain an unbiased estimator of $\ln Z$, is just one way of inferring information about $Z$.

We consider applying different functions $g$ to (\ref{eq:ExpClockZ}); particularly those functions $g$ that transform the exponential distribution to another distribution with known mean. As the original exponential distribution has rate $Z$, the transformed distribution will have mean $f(Z)$, where $f$ will in general no longer be the logarithm function. 
Since we often are interested in estimating various transformations $f(Z)$ of $Z$, this provides us a with a collection of unbiased estimators from which to choose. Moreover, further transforming these estimators yields a collection of (biased) estimators for other transformations of $Z$, including $Z$ itself.

\begin{example}[Weibull tricks]
For any $\alpha > 0$, applying the function $g(x) = x^{\alpha}$ to an $\Exp(\lambda)$ random variable yields a random variable with the $\Weibull(\lambda^{-\alpha}, \alpha^{-1})$ distribution with scale $\lambda^{-\alpha}$ and shape $\alpha^{-1}$, which has mean $\lambda^{-\alpha} \Gamma(1 + \alpha)$ and can be also represented as $\lambda^{-\alpha} W$, where $W \sim \Weibull(1, \alpha^{-1})$. Defining $\{ W(x) \}_{x \in \calX} \stackrel{\text{i.i.d.}}{\sim} \Weibull(1, \alpha^{-1})$ and noting that $g$ is increasing, applying $g$ to the equality in distribution (\ref{eq:ExpClockZ}) gives
\begin{equation}
\min_{x \in \calX} \{ \tilde{p}^{-\alpha} W(x) \}
\sim
\Weibull(Z^{-\alpha}, \alpha^{-1}).
\tag{\ref{eq:ExpClockZ}''}
\label{eq:WeibullZ}
\end{equation}
Estimating the mean of $\Weibull(Z^{-\alpha}, \alpha^{-1})$ yields an unbiased estimator of $Z^{-\alpha} \Gamma(1 + \alpha)$. The special case $\alpha = 1$ corresponds to the identity function $g(x) = x$; we call the resulting trick the \emph{Exponential trick}. \qed
\label{ex:WeibullTricks}
\end{example}

Table~\ref{tab:NewTricks} lists several examples of tricks derived this way. As Example~\ref{ex:WeibullTricks} shows, these tricks may not involve additive perturbation of the potential function $\phi(x)$; the Weibull tricks multiplicatively perturb exponentiated unnormalized probabilities $\tilde{p}^{-\alpha}$ with Weibull noise. As models of interest are often specified in terms of potential functions, to be able to reuse existing MAP solvers in a black-box manner with the new tricks, we seek an equivalent formulation in terms of the potential function. The following Proposition shows that by not passing the function $g$ through the minimization in equation (\ref{eq:ExpClockZ}), the new tricks can be equivalently formulated as averaging additive Gumbel perturbations of the potential function in different spaces.

\newpage

\begin{proposition}
	\label{prop:NewTricksGumbelFormulation}
	For any function $g : [0, \infty) \to \IR$ such that $f(Z) = \IE_{T \sim \Exp(Z)}[g(T)]$ exists, we have
	\begin{equation*}
	f(Z)
	= \IE_{\gamma}\left[ g\left( e^{-c} \exp\left( - \max_{x \in \calX} \{ \phi(x) + \gamma(x) \} \right) \right) \right],
	\end{equation*}
	where $\{ \gamma(x) \}_{x \in \calX} \stackrel{\text{i.i.d.}}{\sim} \operatorname{Gumbel}(-c)$.
	\begin{proof}
		As $\max_x \{ \phi(x) + \gamma(x) \} \sim \Gumbel(-c + \ln Z)$, we have $e^{-c} \exp( \max_x \{ \phi(x) + \gamma(x) \} ) \sim \Exp(Z)$ and the result follows by the assumption relating $f$ and $g$.
	\end{proof}
\end{proposition}

Proposition~\ref{prop:NewTricksGumbelFormulation} shows that the new tricks can be implemented by solving the same MAP problems $\max_x \{ \phi(x) + \gamma(x) \}$ as in the Gumbel trick, and then merely passing the solutions through the function $x \mapsto g(e^{-c} \exp(x))$ before averaging them to approximate the expectation.

\subsection{Comparing Tricks}
\label{sec:ComparingTricks}


\begin{figure}[!t]
	\centering
	\includegraphics[width=\linewidth]{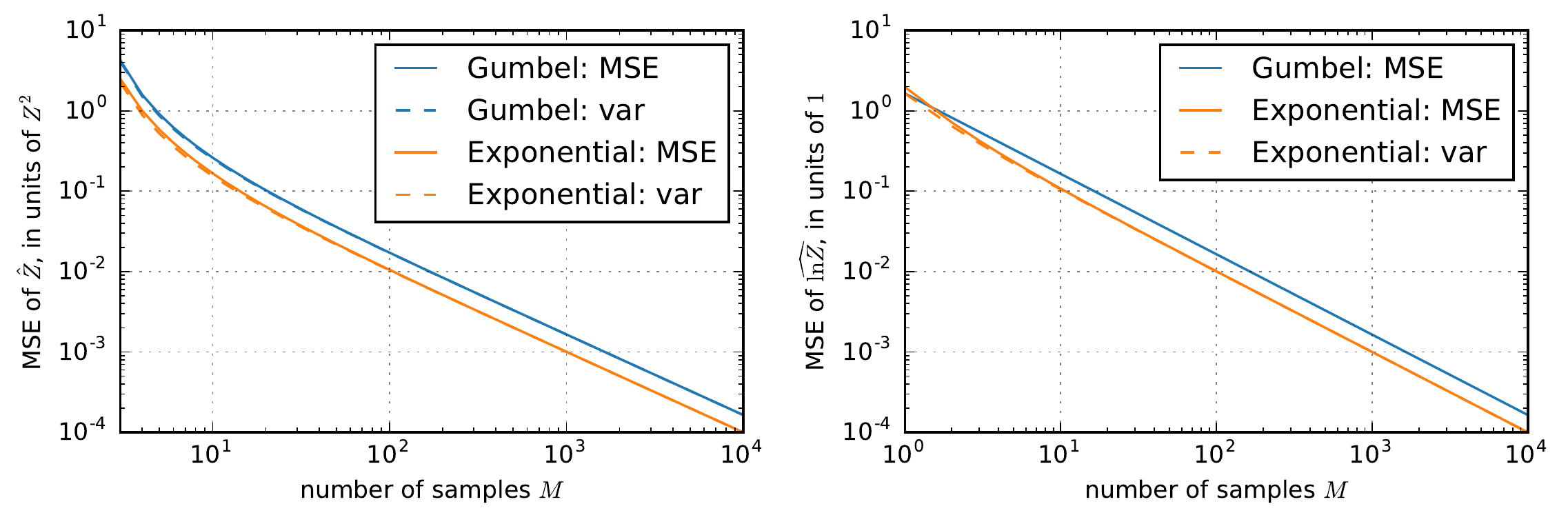}
	\caption{\small{Analytically computed MSE and variance of Gumbel and Exponential trick estimators of $Z$ (left) and $\ln Z$ (right). The MSEs are dominated by the variance, so the dashed and solid lines mostly overlap. See Section~\ref{sec:ComparingTricks:MSE} for details.}}
	\label{fig:estimators_Ms}
	\vspace{-0.1em}
\end{figure}

\subsubsection{Asymptotic efficiency}
\label{sec:ComparingTricks:Delta}

The Delta method~\citep{casella2002statistical} is a simple technique for assessing the asymptotic variance of estimators that are obtained by a differentiable transformation of an estimator with known variance.
The last column in Table~\ref{tab:NewTricks} lists asymptotic variances of corresponding tricks when unbiased estimators of $f(Z)$ are passed through the function $f^{-1}$ to yield (biased, but consistent and non-negative) estimators of $Z$ itself.
It is interesting to examine the constants that multiply $Z^2$ in some of the obtained asymptotic variance expressions for the different tricks. For example, it can be shown using Gurland's ratio~\citep{gurland_1956} that this constant is at least $1$ for the Weibull and Fr\'echet tricks, which is precisely the value achieved by the Exponential trick (which corresponds to $\alpha = 1$). Moreover, the Gumbel trick constant $\pi^2 / 6$ can be shown to be the limit as $\alpha \to 0$ of the Weibull and Fr\'echet trick constants. In particular, the constant of the Exponential trick is strictly better than that of the standard Gumbel trick: $1 < \pi^2 / 6 \approx 1.65$. This motivates us to compare the Gumbel and Exponential tricks in more detail.


\begin{figure}[!t]
	\centering
	\includegraphics[width=\linewidth]{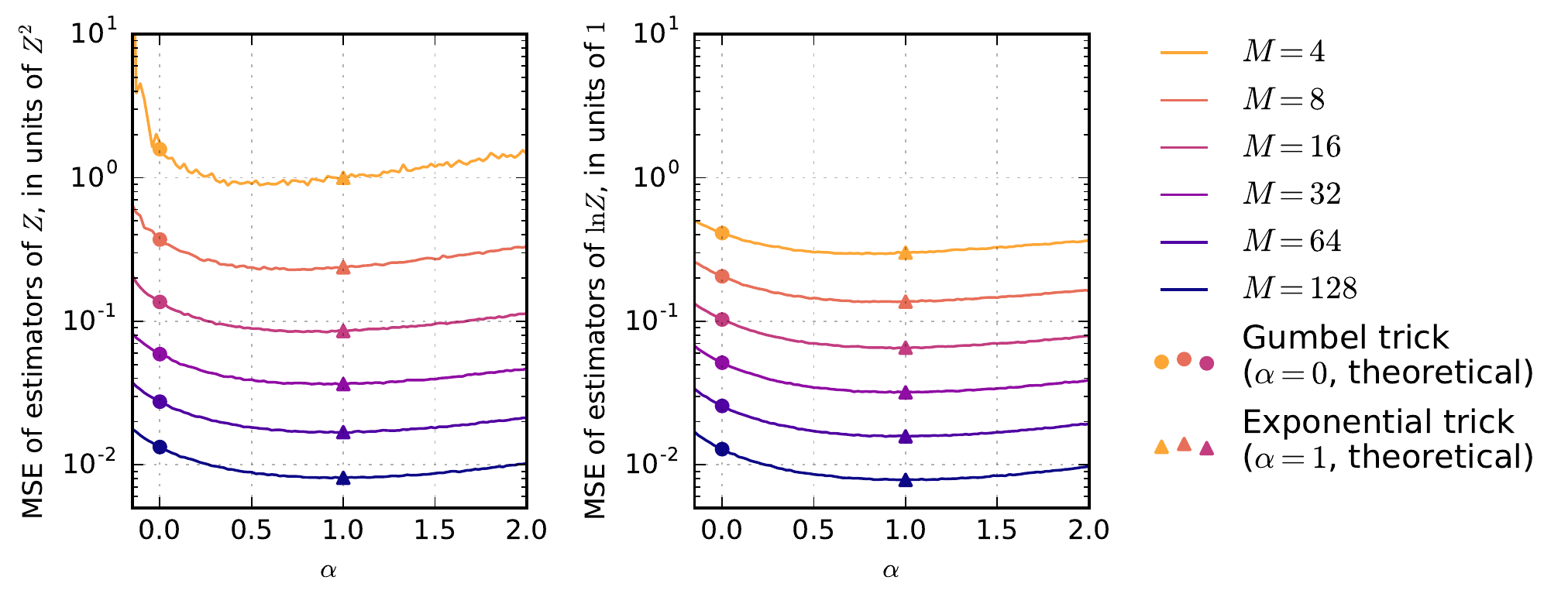}
	\caption{\small{MSE of estimators of $Z$ (left) and $\ln Z$ (right) stemming from Fr\'echet ($-\frac{1}{2} < \alpha < 0$), Gumbel ($\alpha = 0$) and Weibull tricks ($\alpha > 0$). See Section~\ref{sec:ComparingTricks:MSE} for details.}
	}
	\label{fig:estimators_alphas}
	\vspace{-0.9em}
\end{figure}

\subsubsection{Mean squared error (MSE)}
\label{sec:ComparingTricks:MSE}

For estimators $Y$, their $\MSE(Y) = \var(Y) + \bias(Y)^2$ is a commonly used comparison metric. When the Gumbel or Exponential tricks are used to estimate either $Z$ or $\ln Z$, the biases, variances, and MSEs of the estimators can be computed analytically using standard methods (Appendix~\ref{app:sec:FullRank}).

For example, the unbiased estimator of $\ln Z$ from the Gumbel trick can be turned into a consistent non-negative estimator of $Z$ by exponentiation: $Y = \exp( \frac{1}{M} \sum_{m = 1}^M X_m )$, where $X_1, \ldots, X_M \stackrel{\text{i.i.d.}}{\sim} \Gumbel(-c + \ln Z)$ are obtained using equation (\ref{eq:GumbelZ}). The bias and variance of $Y$ can be computed using independence and the moment generating functions of the $X_m$'s, see Appendix~\ref{app:sec:FullRank} for details.

Perhaps surprisingly, all estimator properties only depend on the true value of $Z$ and not on the structure of the model (distribution $p$), since the estimators rely only on i.i.d.~samples of a $\Gumbel(-c + \ln Z)$ random variable. Figure~\ref{fig:estimators_Ms} shows the analytically computed estimator variances and MSEs. For estimating $Z$ itself (left), the Exponential trick outperforms the Gumbel trick in terms of MSE for all sample sizes $M \ge 3$ (for $M \in \{ 1, 2 \}$, both estimators have infinite variance and MSE). The ratio of MSEs quickly approaches $\pi^2 / 6$, and in this regime the Exponential trick requires $1-6/\pi^2 \approx 39\%$ fewer samples than the Gumbel trick to reach the same MSE.
Also, for estimating $\ln Z$, (Figure~\ref{fig:estimators_Ms}, right), the Exponential trick provides a lower MSE estimator for sample sizes $M \ge 2$; only for $M = 1$ the Gumbel trick provides a better estimator.

Note that as biases are available analytically, the estimators can be easily debiased (by subtracting their bias). One then obtain estimators with MSEs equal to the variances of the original estimators, shown dashed in Figure~\ref{fig:estimators_Ms}. The Exponential trick would then always outperform the Gumbel trick when estimating $\ln Z$, even with sample size $M = 1$.


For Weibull tricks with $\alpha \not= 1$ and Fr\'echet tricks, we estimated the biases and variances of estimators of $Z$ and $\ln Z$ by constructing $K = 100,000$ estimators in each case and evaluating their bias and variance. Figure~\ref{fig:estimators_alphas} shows the results for varying $\alpha$ and several sample sizes $M$. We plot the analytically computed value for the Gumbel trick at $\alpha = 0$, as we observe that the Weibull trick interpolates between the Gumbel trick and the Exponential trick as $\alpha$ increases from $0$ to $1$. We note that the minimum MSE estimator is obtained by choosing a value of $\alpha$ that is close to $1$, i.e. the Exponential trick. This agrees with the finding from Section~\ref{sec:ComparingTricks:Delta} that $\alpha = 1$ is optimal as $M \to \infty$.


\subsection{Bayesian Perspective}

A Bayesian approach exposes two choices when constructing estimators of $Z$, or of its transformations $f(Z)$:
\begin{enumerate}
	\item A choice of prior distribution $p_0(Z)$, encoding prior beliefs about the value of $Z$ before any observations.
	\item A choice of how to summarize the posterior distribution $p_M(Z | X_1, \ldots, X_M)$ given $M$ samples.
\end{enumerate}

Taking the Jeffrey's prior $p_0(Z) \propto Z^{-1}$, an improper prior that it is invariant under reparametrization, observing $M$ samples $X_1, \ldots, X_M \stackrel{\text{i.i.d.}}{\sim} \Exp(Z)$ yields the posterior:
\begin{equation*}
p_M(Z | X_1, \ldots, X_M)
\propto Z^{M-1} e^{-Z \sum_{m = 1}^M X_m}.
\end{equation*}
Recognizing the density of a $\operatorname{Gamma}(M, \sum_{m = 1}^M X_m)$ random variable, the posterior mean is
\begin{equation*}
\IE[Z | X_1, \ldots, X_M]
= \frac{M}{\sum_{m = 1}^M X_m}
= \left( \frac{1}{M} \sum_{m = 1}^M X_m \right)^{-1},
\end{equation*}
coinciding with the Exponential trick estimator of $Z$.

\section{Low-rank Perturbations}
\label{sec:LowRank}

One way of exploiting perturb-and-MAP to yield computational savings is to replace independent perturbations of each configuration's potential with an approximation. Such approximations are available e.g.~in discrete graphical models, where the sampling space $\calX$ has a product space structure $\calX = \calX_1 \times \cdots \times \calX_n$, with $\calX_i$ the state space of the $i$-th variable.


\begin{definition}[ \cite{hazan_partition_2012}]
	\label{def:UnaryPerturbationMAP}
	The \emph{sum-unary perturbation MAP value} is the random variable
	\begin{equation*}
		U := \max_{\bx \in \calX} \Big\{ \phi(\bx) + \sum_{i = 1}^n \gamma_i(x_i) \Big\},
	\end{equation*}
	where $\{ \gamma_i(x_i) \mid x_i \in \calX_i, 1 \leq i \leq n \} \stackrel{\text{i.i.d}}{\sim} \operatorname{Gumbel}(-c)$.
\end{definition}
This definition involves $|\calX_1| + \cdots + |\calX_n|$ i.i.d.~Gumbel random variables, rather than $|\calX|$. (With $n = 1$ this coincides with full-rank perturbations and $U \sim \Gumbel(-c + \ln Z)$.) For $n > 2$ the distribution of $U$ is not available analytically. One can similarly define the \emph{pairwise} (or higher-order) \emph{perturbations}, where independent Gumbel noise is added to each pairwise (or higher-order) potential.

Unary perturbations provide the upper bound $\ln Z \leq \IE[U]$ on the log partition function~\citep{hazan_partition_2012}, can be used to construct a sequential sampler for the Gibbs distribution~\citep{hazan_sampling_2013}, and, if the perturbations are scaled down by a factor of $n$, a lower bound on $\ln Z$ can also be recovered~\citep{hazan_sampling_2013}.
In this section we show that a subfamily of tricks introduced in Section~\ref{sec:NewTricks}, consisting of Fr\'echet and Weibull (and Exponential) tricks, is applicable in the low-rank perturbation setting and use them to derive new families of upper and lower bounds on $
\ln Z$ and sequential samplers for the Gibbs distribution. Please note full proofs are deferred to  Appendix~\ref{app:sec:UnarySum} and~\ref{app:sec:UnaryAvg}.

\subsection{Upper Bounds on the Partition Function}
\label{sec:LowRank:UpperBounds}

The following family of upper bounds on $\ln Z$ can be derived from the Fr\'echet and Weibull tricks.
\begin{proposition}
	\label{prop:LowRankUpperBoundOnLnZ}
	For any $\alpha \in (-1, 0) \cup (0, \infty)$, the upper bound $\ln Z \leq \calU(\alpha)$ holds with
	\begin{equation*}
	\calU(\alpha) := n \frac{\ln \Gamma(1+\alpha)}{\alpha} + nc - \frac{1}{\alpha} \ln \IE_{\gamma}\left[ e^{-\alpha U} \right].
	\end{equation*}
	\begin{proof} (Sketch.) By induction on $n$, with the induction step provided by our Clamping Lemma (Lemma~\ref{lem:ClampingLemma}) below.
	\end{proof}
\end{proposition}
To evaluate these bounds in practice, $\IE[e^{-\alpha U}]$ is estimated using samples of $U$. Corollary 9 of~\citet{hazan_high_2016} can be used to show that $\var(e^{-\alpha U})$ is finite for $\alpha > - \frac{1}{2\sqrt{n}}$, and so then the estimation is well-behaved.

A natural question is how these new bounds relate to the Gumbel trick upper bound $\ln Z \leq \IE[U]$ by~\citet{hazan_partition_2012}. The following result aims to answers this:
\begin{proposition}
\label{prop:LowRank:Ulimit0}
The limit of $\calU(\alpha)$ as $\alpha \to 0$ exists and equals $\calU(0) := \IE[U]$, i.e. the Gumbel trick upper bound.
\end{proposition}
The question remains: When is it advantageous to use a value $\alpha \not= 0$ to obtain a tighter bound on $\ln Z$ than the Gumbel trick bound?
The next result can provide guidance:
\begin{proposition}
\label{prop:UpperBoundDerivativeAt0}
The function $\calU(\alpha)$ is differentiable at $\alpha = 0$ and the derivative equals
\begin{equation*}
\frac{\d}{\d \alpha} \calU(\alpha) \bigg\rvert_{\alpha = 0}
= \frac{1}{2} \left( n \frac{\pi^2}{6} - \var(U) \right).
\end{equation*}
\end{proposition}
While the variance of $U$ is generally not tractable, in practice one obtains samples from $U$ to estimate the expectation in $\calU(\alpha)$ and these samples can be reused to assess $\var(U)$. Interestingly, $\var(U)$ equals $n \pi^2 / 6$ for both the uniform distribution and the distribution concentrated on a single configuration, and in our empirical investigations always $\var(U) \leq n \pi^2 / 6$. Then the derivative at $0$ is non-negative and Fr\'echet tricks provide tighter bounds on $\ln Z$. However, as $\calU(\alpha)$ is estimated with samples, the question of estimator variance arises. We investigate the trade-off between tightness of the bound $\ln Z \leq \calU(\alpha)$ and the variance incurred in estimating $\calU(\alpha)$ empirically in Section~\ref{sec:exp:LowRankBounds}.

\subsection{Clamping}
\label{sec:LowRank:Clamping}

Consider the \emph{partial sum-unary perturbation MAP values}, where the values of the first $j - 1$ variables have been fixed, and only the rest are perturbed:
\begin{equation*}
	U_j(x_1, \ldots, x_{j - 1}) := \max_{x_j, \ldots, x_n} \left\{ \phi(\bx) + \sum_{i = j}^n \gamma_i(x_i) \right\}.
\end{equation*}
The following lemma involving the $U_j$'s serves three purposes: (I.) it provides the induction step for Proposition~\ref{prop:LowRankUpperBoundOnLnZ}, (II.) it shows that clamping never hurts partition function estimation with Fr\'echet and Weibull tricks, and (III.) it will be used to show that a sequential sampler constructed in Section~\ref{sec:lowRank:sampler} below is well-defined.

\begin{lemma}[Clamping Lemma]\label{lem:clamping}
For any $j \in \{ 1, \ldots, n \}$ and $(x_1, \ldots, x_{j - 1}) \in \calX_1 \times \cdots \times \calX_{j - 1}$, the following inequality holds with any $\alpha \in (-1, 0) \cup (0, \infty)$:
\begin{align*}
\sum_{x_j \in \calX_j} \IE_{\gamma}\left[ e^{-(n-j) \ln \Gamma(1 + \alpha) - \alpha (n-j) c)} e^{-\alpha U_{j + 1}} \right]^{-1/\alpha}
\\ \leq
\IE_{\gamma}\left[ e^{-(n-(j-1)) \ln \Gamma(1 + \alpha) - \alpha (n-(j-1)) c)} e^{-\alpha U_j} \right]^{-1/\alpha}
\end{align*}
\begin{proof} This follows directly from the Fr\'echet trick ($\alpha \in (-1, 0)$) or the Weibull trick ($\alpha > 0$) and representing the Fr\'echet resp. Weibull random variables in terms of Gumbel random variables. See Appendix~\ref{app:sec:sum_unary:upper} for more details.
\end{proof}
\label{lem:ClampingLemma}
\end{lemma}


\begin{corollary}\label{cor:clamp}
Clamping never hurts $\ln Z$ estimation using any of the Fr\'echet or Weibull upper bounds $\calU(\alpha)$.
\end{corollary}
\begin{proof} Applying the function $x \mapsto \ln(x)$ to both sides of the Clamping Lemma~\ref{lem:ClampingLemma} with $j = 1$, the right-hand side equals $\calU(\alpha)$, while the left-hand side is the estimate of $\ln Z$ after clamping variable $x_1$.
\end{proof}
This was shown previously in restricted settings \citep{hazan_sampling_2013,Zhao16}. 
Similar results showing that clamping improves partition function estimation have been obtained for the mean field and TRW approximations \cite{WelDom16}, and in certain settings for the Bethe approximation \cite{zbn} and \textsc{L-Field}~\citep{Zhao16}.

\subsection{Sequential Sampling}
\label{sec:lowRank:sampler}

\citet{hazan_sampling_2013} derived a sequential sampling procedure for the Gibbs distribution by exploiting the $\calU(0)$ Gumbel trick upper bound on $\ln Z$. In the same spirit, one can derive sequential sampling procedures from the Fr\'echet and Weibull tricks, leading to the following algorithm.

\begin{algorithm}[H]
\begin{algorithmic}[1]
\REQUIRE $\alpha \in (-1, 0) \cup (0, \infty)$, potential function $\phi$ on $\calX$
\ENSURE a sample $\bx$ from the Gibbs distribution $\propto e^{\phi(\bx)}$
\FOR{$j = 1$ \textbf{to} $n$}
\FOR{$x_j \in \calX_j$}
\STATE $p_j(x_j) \gets \frac{e^{-c}}{\Gamma(1 + \alpha)^{1/\alpha}} \frac{\IE_{\gamma}\left[ e^{-\alpha U_{j + 1}(x_1, \ldots, x_j)} \right]^{-1/\alpha}}{\IE_{\gamma}\left[ e^{-\alpha U_j(x_1, \ldots, x_{j-1})} \right]^{-1/\alpha}}$
\ENDFOR
\STATE $p_j(\text{reject)} \gets 1 - \sum_{x_j \in \calX_j} p_j(x_j)$
\STATE $x_j \gets $ sample according to $p_j$
\IF{$x_j == \text{reject}$}
\STATE \textsc{Restart} (goto 1)
\ENDIF
\ENDFOR
\end{algorithmic}
\caption{Sequential sampler for Gibbs distribution}
\label{alg:SequentialSampler}
\end{algorithm}

This algorithm is well-defined if $p_j(\text{reject}) \geq 0$ for all $j$, which can be shown by canceling terms in the Clamping Lemma~\ref{lem:ClampingLemma}. We discuss correctness in Appendix~\ref{app:sec:sum_unary:sampler}.
As for the Gumbel sequential sampler of~\citet{hazan_sampling_2013}, the expected number of restarts (and hence the running time) only depend on the quality of the upper bound $(\calU(\alpha) - \ln Z)$, and not on the ordering of variables.

\subsection{Lower Bounds on the Partition Function}

Similarly as in the Gumbel trick case~\citep{hazan_sampling_2013}, one can derive lower bounds on $\ln Z$ by perturbing an arbitrary subset $S$ of variables.

\begin{proposition}
	\label{prop:LowRankLowerBoundAlphaSubset}
	Let $\calX = \calX_1 \times \cdots \calX_n$ be a product space and $\phi$ a potential function on $\calX$. Let $\alpha \in (-1, 0) \cup (0, \infty)$. For any subset $S \subseteq \{ 1, \ldots, n \}$ of the variables $x_1, \ldots, x_n$ we have $\ln Z \geq$
	\begin{equation*}
	c + \frac{\ln \Gamma(1 + \alpha)}{\alpha} - \frac{1}{\alpha} \ln \IE\left[ e^{- \alpha \max_{\bx} \{ \phi(\bx) + \gamma_{S} (\bx_{S}) \} } \right],
	\end{equation*}
	where $\bx_{S} := \{ x_i : i \in S \}$ and $\gamma_{S} (\bx_{S}) \sim \Gumbel(-c)$ independently for each setting of $\bx_{S}$.
\end{proposition}

By averaging $n$ such lower bounds corresponding to singleton sets $S = \{ i \}$ together, we obtain a lower bound on $\ln Z$ that involves the \emph{average-unary perturbation MAP value}
\begin{equation*}
L := \max_{\bx \in \calX} \left\{ \phi(\bx) + \frac{1}{n} \sum_{i = 1}^n \gamma_i(x_i) \right\}.
\end{equation*}

\begin{corollary}
	\label{corl:LowRankLowerBound}
	For any $\alpha \in (-1, 0) \cup (0, \infty)$, we have the lower bound $\ln Z \geq \calL(\alpha)$, where
	\begin{equation*}
	\calL(\alpha) :=
	c + \frac{\ln \Gamma(1 + \alpha)}{\alpha} - \frac{1}{n \alpha} \ln \IE\left[ \exp\left( - n \alpha L \right) \right].
	\end{equation*}
\end{corollary}

Again, $\calL(0) := \IE[L]$ can be defined by continuity, where $\IE[L] \leq \ln Z$ is the Gumbel trick lower bound by~\citet{hazan_sampling_2013}.

\section{Advantages of the Gumbel Trick}
\label{sec:LowRankErrors}

We have seen how the Gumbel trick can be embedded into a continuous family of tricks, consisting of Fr\'echet, Exponential, and Weibull tricks. We showed that the new tricks can provide more efficient estimators of the partition function in the full-rank perturbation setting (Section~\ref{sec:NewTricks}), and in the low-rank perturbation setting lead to sequential samplers and new bounds on $\ln Z$, which can be also more efficient, as we investigate in Section~\ref{sec:exp:LowRankBounds}. To balance the discussion of merits of different tricks, in this section we briefly highlight advantages of the Gumbel trick that stem from its simpler analytical form.

First, by consulting Table~\ref{tab:NewTricks} we see that the function $g(x) = -\ln x - c$ has the property that the variance of the resulting estimator (of $\ln Z$) does not depend on the value of $Z$; the function $g$ is a variance stabilizing transformation for the Exponential distribution.

Second, exploiting the fact that the logarithm function leads to additive perturbations,~\citet{maji_active_2014} showed that the entropy of $x^{*}$, the configuration with maximum potential after sum-unary perturbation in the sense of Definition~\ref{def:UnaryPerturbationMAP}, can be bounded as $H(x^{*}) \leq B(p) := \sum_{i = 1}^n \IE_{\gamma_i}\left[ \gamma_i(x^{*}_i) \right]$. We extend this result to show how the errors of bounding $\ln Z$, sampling, and entropy estimation are related:

\begin{proposition}
	\label{prop:UnarySumGapsLink}
	Writing $p$ for the Gibbs distribution and $B(p) := \IE_{\gamma_i}\left[ \gamma_i(x^{*}_i) \right]$ for the entropy bound, we have
	\begin{align*}
	&
	\underbrace{(\calU(0) - \ln Z)}_{\text{error in } \ln Z \text{ bound}}
	+
	\underbrace{\KL{x^{*}}{p}}_{\text{sampling error}}
	&=
	\underbrace{B(p) - H(x^{*})}_{\text{error in entropy estimation}}.
	\end{align*}
\end{proposition}

Third, the additive character of the Gumbel perturbations can also be used to derive a new result relating the error of the lower bound $\calL(0)$ and of sampling $x^{**}$ as the configuration achieving the maximum average-unary perturbation value $L$, instead of sampling from the Gibbs distribution $p$:

\begin{proposition}
	\label{prop:UnaryAvgGapsLink}
	Writing $p$ for the Gibbs distribution,
	\begin{equation*}
	\underbrace{\ln Z - \calL(0)}_{\text{error in } \ln Z \text{ bound}}
	\geq
	\underbrace{\KL{x^{**}}{p}}_{\text{sampling error}}
	\geq 0.
	\end{equation*}
	\begin{remark}
		While we knew from~\citet{hazan_sampling_2013} that $\ln Z - \calL(0) \geq 0$, this is a stronger result showing that the size of the gap is an upper bound on the KL divergence between the approximate sampling distribution of $x^{**}$ and the Gibbs distribution $p$.
	\end{remark}
\end{proposition}

Proofs of the new results appear in Appendix~\ref{app:sec:sum_unary:errors} and~\ref{app:sec:sum_average:errors}.

Fourth, viewed as a function of the Gumbel perturbations $\gamma$, the random variable $U$ has a bounded gradient, allowing earlier measure concentration results \citep{orabona_measure_2014,hazan_high_2016}. Proving similar measure concentration results for the expectations $\IE[e^{-\alpha U}]$ appearing in $\calU(\alpha)$ for $\alpha \not= 0$ may be more challenging.

\section{Experiments}
\label{sec:Experiments}

We conducted experiments with the following aims:
\begin{enumerate}
	\item To show that the higher efficiency of the Exponential trick in the full-rank perturbation setting is useful in practice, we compared it to the Gumbel trick in A* sampling \citep{maddison_ast_2014} (Section~\ref{sec:exp:Astar}) and in the large-scale discrete sampling setting of~\citet{chen_scalable_2016} (Section~\ref{sec:exp:FullRank}).
	\item To show that non-zero values of $\alpha$ can lead to better estimators of $\ln Z$ in the low-rank perturbation setting as well, we compare the Fr\'echet and Weibull trick bounds $\calU(\alpha)$ to the Gumbel trick bound $\calU(0)$ on a common discrete graphical model with different coupling strengths; see Section~\ref{sec:exp:LowRankBounds}.
\end{enumerate}

\subsection{A* Sampling}
\label{sec:exp:Astar}

A* sampling~\citep{maddison_ast_2014} is a sampling algorithm for continuous distributions that perturbs the log-unnormalized density $\phi$ with a continuous generalization of the Gumbel trick, called the Gumbel process, and uses a variant of A* search to find the location of the maximum of the perturbed $\phi$. Returning the location yields an exact sample from the original distribution, as in the discrete Gumbel trick. Moreover, the corresponding maximum value also has the $\Gumbel(-c+\ln Z)$ distribution \citep{maddison_ast_2014}. Our analysis in Section~\ref{sec:ComparingTricks} tells us that the Exponential trick yields an estimator with lower MSE than the Gumbel trick; we briefly verified this on the Robust Bayesian Regression experiment of \citet{maddison_ast_2014}. We constructed estimators of $\ln Z$ from the Gumbel and Exponential tricks (debiased version, see Section~\ref{sec:ComparingTricks:MSE}), and assessed their variances by constructing each estimator $K = 1000$ times and looking at the sample variance. Figure~\ref{fig:exp:full_rank_astar} shows that the Exponential trick requires up to 40\% fewer samples to reach a given MSE.

\subsection{Scalable Partition Function Estimation}
\label{sec:exp:FullRank}

\citet{chen_scalable_2016} considered sampling from a discrete distribution of the form $p(x) \propto f_0(x) \prod_{s = 1}^S f_s(x)$ when the number of factors $S$ is large relative to the sample space size $|\calX|$. Computing i.i.d.~Gumbel perturbations $\gamma(x)$ for each $x \in \calX$ is then relatively cheap compared to evaluating all potentials $\phi(x) = f_0(x) + \sum_{s = 1}^S \ln f_s(x)$. \citet{chen_scalable_2016} observed that each (perturbed) potential can be estimated by subsampling the factors, and potentials that appear unlikely to yield the MAP value can be pruned off from the search early on. The authors formalized the problem as a Multi-armed bandit problem with a finite reward population and derived approximate algorithms for efficiently finding the maximum perturbed potential with a probabilistic guarantee.

While \citet{chen_scalable_2016} considered sampling, by modifying their procedure to return the value of the maximum perturbed potential rather than the argmax (cf equations (\ref{eq:ExpClockZ}) and (\ref{eq:ExpClockSample})), we can estimate the partition function instead. However, the approximate algorithm only guarantees to find the MAP configuration with a probability $1 - \delta$. Figure~\ref{fig:exp:full_rank_bandits} shows the results of running the Racing-Normal algorithm of \citet{chen_scalable_2016} on the synthetic dataset considered by the authors with the ``very hard" noise setting $\sigma = 0.1$. For low error bounds $\delta$ the Exponential trick remained close to optimal, but for a larger error bound the Weibull trick interpolation between the Gumbel and Exponential tricks proved useful to provide an estimator with lower MSE.

\begin{figure}[!t]
	\centering
	\begin{subfigure}{.5\linewidth}
		\includegraphics[width=\linewidth]{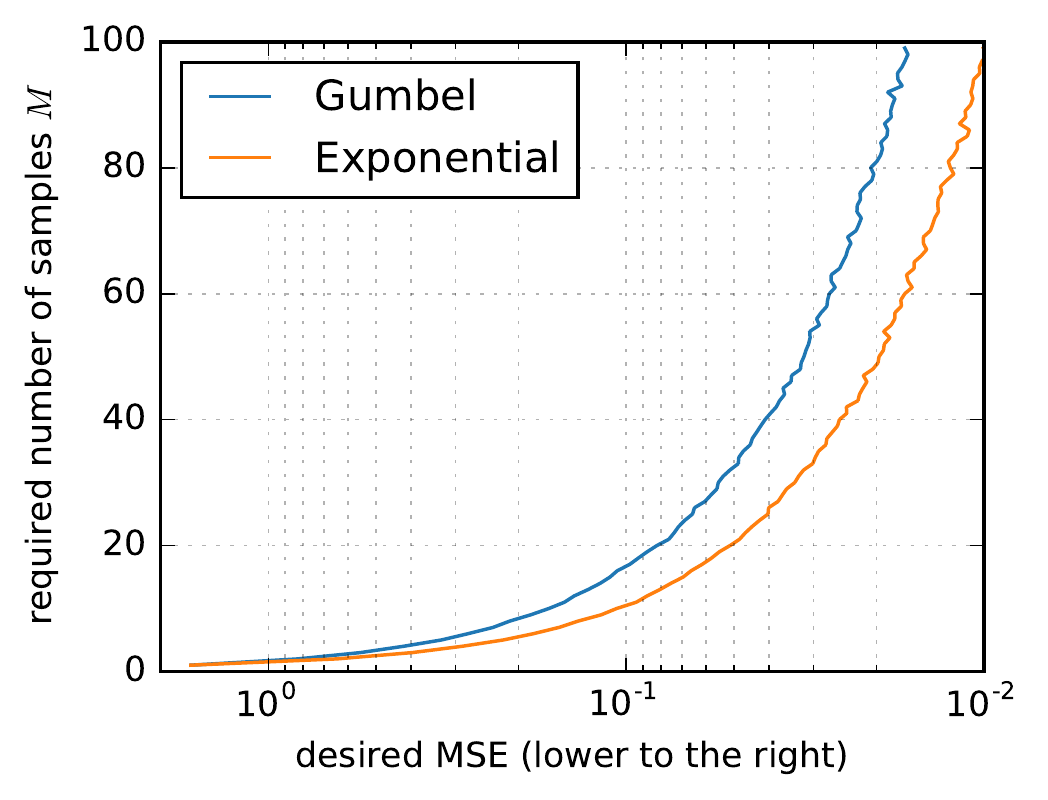}
		\caption{}
		\label{fig:exp:full_rank_astar}
	\end{subfigure}%
	\begin{subfigure}{.5\linewidth}
		\includegraphics[width=\linewidth]{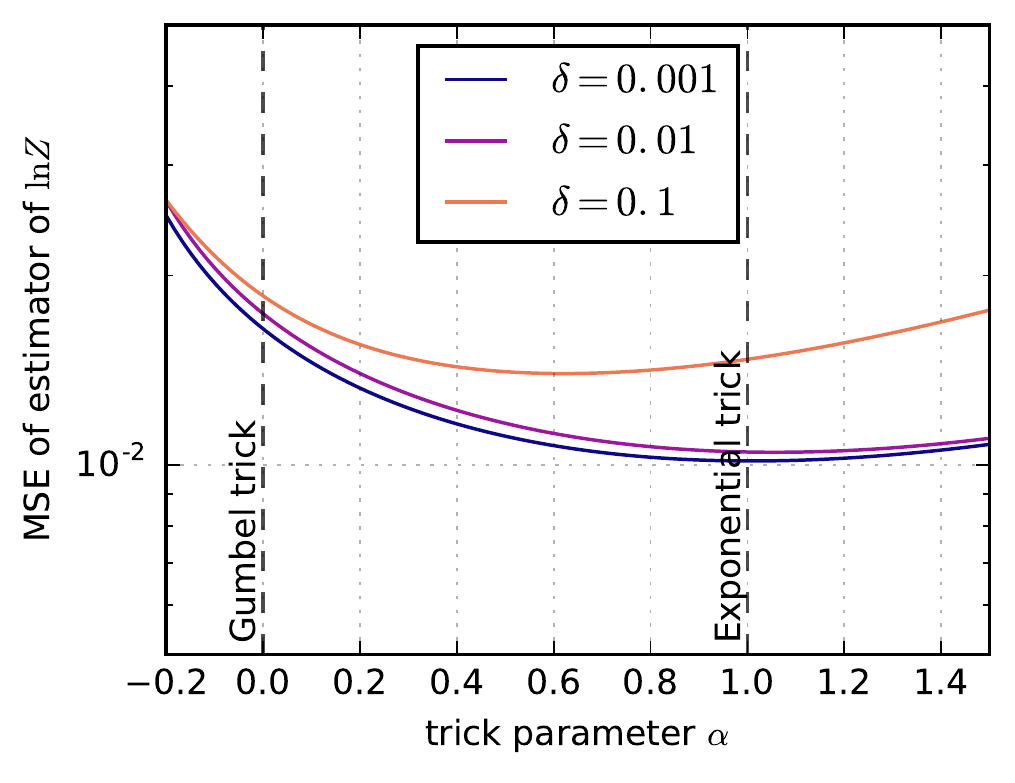}
		\caption{}
		\label{fig:exp:full_rank_bandits}
	\end{subfigure}
	\caption{\small{(a) Sample size $M$ required to reach a given MSE using Gumbel and Exponential trick estimators of $\ln Z$, using samples from \emph{$A^{*}$ sampling} (see Section~\ref{sec:exp:Astar}) on a Robust Bayesian Regression task. The Exponential trick is more efficient, requiring up to 40\% fewer samples to reach a given MSE.
	(b) MSE of $\ln Z$ estimators for different values of $\alpha$, using $M = 100$ samples from the approximate MAP algorithm discussed in Section~\ref{sec:exp:FullRank}, with different error bounds $\delta$. For small $\delta$, the Exponential trick is close to optimal, matching the analysis of Section~\ref{sec:ComparingTricks:MSE}. For larger $\delta$, the Weibull trick interpolation between the Gumbel and Exponential tricks can provide an estimator with lower MSE.
	\vspace{-0.75em}
	}}
	\label{fig:exp:full_rank_alpha}
\end{figure}

\subsection{Low-rank Perturbation Bounds on $\ln Z$}
\label{sec:exp:LowRankBounds}

\citet{hazan_partition_2012} evaluated tightness of the Gumbel trick upper bound $\calU(0) \geq \ln Z$ on $10 \times 10$ binary spin glass models. We show one can obtain more accurate estimates of $\ln Z$ on such models by choosing $\alpha \not= 0$. To account for the fact that in practice an expectation in $\calU(\alpha)$ is replaced with a sample average, we treat $\calU(\alpha)$ as an estimator of $\ln Z$ with asymptotic bias equal to the bound gap $(\calU(\alpha) - \ln Z)$, and estimate its MSE.

Figure~\ref{fig:exp:spin_glass_MSE} shows the MSEs of $\calU(\alpha)$ as estimators of $\ln Z$ on $10 \times 10$ ($n = 100$) binary pairwise grid models with unary potentials sampled uniformly from $[-1, 1]$ and pairwise potentials from $[0, C]$ (\emph{attractive} models) or from $[-C, C]$ (\emph{mixed} models), for varying coupling strengths $C$. We replaced the expectations in $U(\alpha)$'s with sample averages of size $M = 100$, using \mbox{libDAI}~\citep{Mooij_libDAI_10} to solve the MAP problems yielding these samples. We constructed each estimator $1000$ times to assess its variance.

\begin{figure}[!t]
	\centering
	\includegraphics[width=\linewidth]{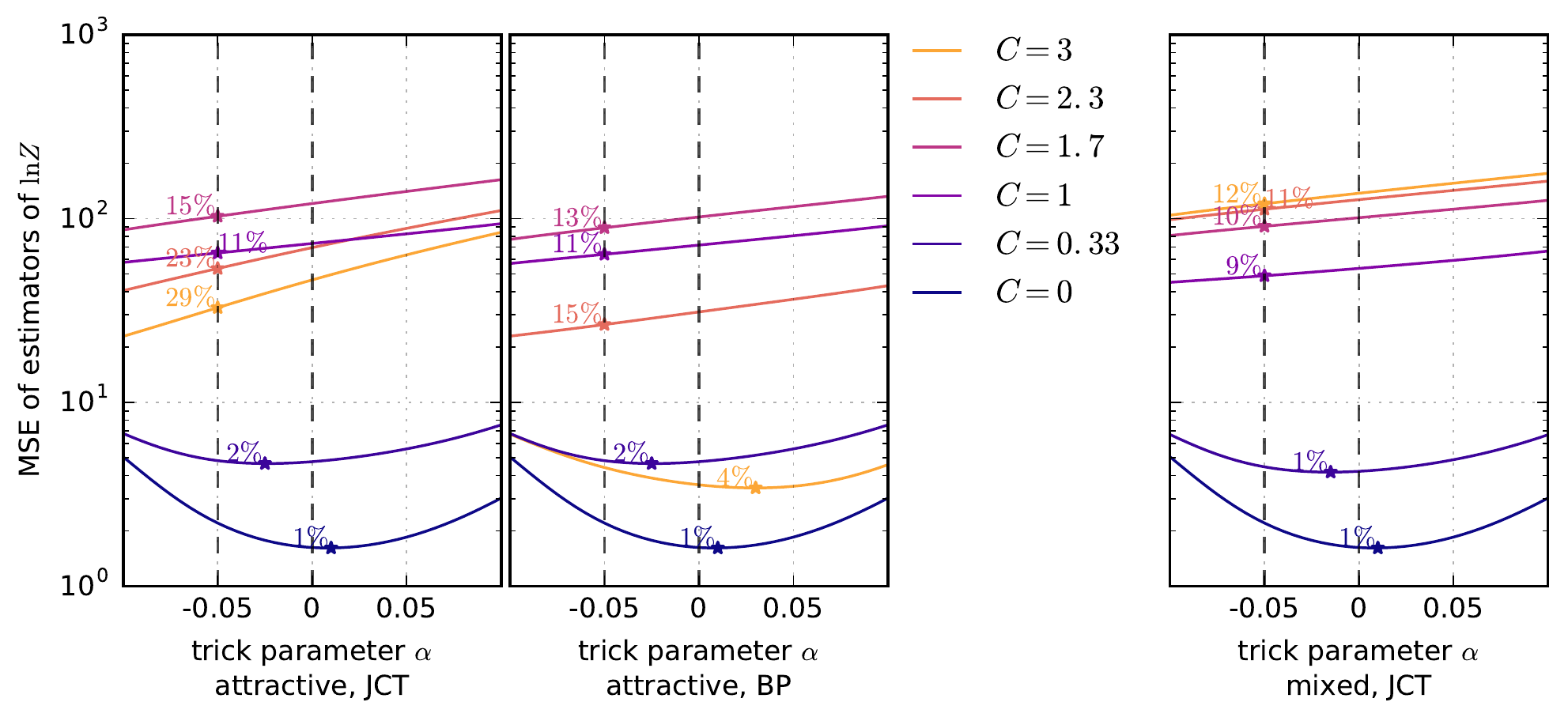}
	\caption{\small{MSEs of $\calU(\alpha)$ as estimators of $\ln Z$ on $10 \times 10$ attractive (left, middle) and mixed (right) spin glass model with different coupling strengths $C$ (see Section~\ref{sec:exp:LowRankBounds}). We also show the percentage of samples saved by using the best $\alpha$ in place of the Gumbel trick estimator $\calU(0)$, assuming the asymptotic regime. For this we only considered $\alpha > -1/(2\sqrt{n}) = -0.05$, where variance is provably finite, see Section~\ref{sec:LowRank:UpperBounds}. The MAP problems were solved using the exact junction tree algorithm (JCT, left and right), or approximate belief propagation (BP, middle). 
In all cases, when coupling is very low, $\alpha$ close to $0$ is optimal. This also holds for BP when coupling is high. 
In other regimes, upper bounds for the Fr\'echet trick, i.e.~$\alpha < 0$, provide more accurate estimators.
	\vspace{-0.5em}
	}}
	\label{fig:exp:spin_glass_MSE}
\end{figure}

\section{Discussion}
\label{sec:Discussion}

By casting partition function evaluation as a parameter estimation problem for the exponential distribution, we derived a family of methods of which the Gumbel trick is a special case. These methods can be equivalently seen as (1) perturbing models using different distributions, or as (2) averaging standard Gumbel perturbations in different spaces, allowing implementations with little additional cost.

We showed that in the full-rank perturbation setting, the new Exponential trick provides an estimator with lower MSE, or instead allows using up to 40\% fewer samples than the Gumbel trick estimator to reach the same MSE.

In the low-rank perturbation setting, we used our Fr\'echet, Exponential and Weibull tricks to derive new bounds on $\ln Z$ and sequential samplers for the Gibbs distribution, and showed that these can also behave better than the corresponding Gumbel trick results. However, the optimal trick to use (as specified by $\alpha$) depends on the model, sample size, and MAP solver used (if approximate). Since in practice the dominant computational cost is carried by solving repeated instances of the MAP problem, one can try and assess different values of $\alpha$ on the problem at hand. That said, we believe that investigating when different tricks yield better results is an interesting avenue for future work.

Finally, we balanced the discussion by pointing out that the Gumbel trick has a simpler analytical form which can be exploited to derive more interesting theoretical statements in the low-rank perturbation setting. Beyond existing results, we derived new connections between errors of different procedures using low-rank Gumbel perturbations.

\section*{Acknowledgements}
The authors thank Tamir Hazan for helpful discussions, and Mark Rowland, Maria Lomeli, and the anonymous reviewers for helpful comments. AW acknowledges support by the Alan Turing Institute under EPSRC grant EP/N510129/1, and by the Leverhulme Trust via the CFI.

\begin{small}
\bibliography{gumbel_ICML}
\bibliographystyle{icml2017}
\end{small}

\clearpage
\newpage
\onecolumn
\appendix

\input{gumbel_appendix}

\end{document}

%% file: gumbel_appendix.tex
\section*{APPENDIX: Lost Relatives of the Gumbel Trick}

Here we provide proofs for the results stated in the main text, together with additional supporting lemmas required for these proofs.

\section{Comparison of Gumbel and Exponential tricks}
\label{app:sec:FullRank}

In Section~\ref{sec:ComparingTricks:Delta} we analyzed the asymptotic efficiency of different estimators of $Z$ by measuring their asymptotic variance. (As all our estimators in the full-rank perturbation setting are consistent, their bias is $0$ in the limit of infinite data, and so this asymptotic variance equals the asymptotic MSE.) In the non-asymptotic regime, where an estimator $\hat{Z}$ is constructed from a finite set of $M$ samples, we can analyze both the variance $\var(\hat{Z})$ and the bias $(\IE[\hat{Z}] - Z)$ of the estimator. While in most cases these cannot be obtained analytically and there we can resort to an empirical evaluation, for the estimators stemming from the Gumbel and Exponential tricks analytical treatment turns out to be possible using standard methods.

\subsection{Estimating $Z$}

\paragraph{Gumbel trick} The Gumbel trick yields an unbiased estimator for $\ln Z$, and we can turn it into a consistent estimator of $Z$ by exponentiating it:
\begin{equation*}
\hat{Z}
:= \exp\left( \frac{1}{M} \sum_{m = 1}^M X_m \right)
\hspace{2em}\text{where}\hspace{2em}
X_1, \ldots, X_M \stackrel{\text{iid}}{\sim} \Gumbel(- c + \ln Z)
.
\end{equation*}
Recalling that the moment generating function of a $\Gumbel(\mu)$ distribution is $G(t) = \Gamma(1 - t) e^{\mu t}$, we can obtain by using independence of the samples:
\begin{align*}
\IE[\hat{Z}] &=
\prod_{m = 1}^M \IE[e^{X_m / M}]
= \left( \Gamma(1 - 1/M) e^{(\ln Z - c) / M} \right)^M
= \Gamma(1 - 1/M)^M e^{- c} Z
,
\\
\IE[\hat{Z}^2] &=
\prod_{m = 1}^M \IE[e^{2 X_m / M}]
= \left( \Gamma(1 - 2/M) e^{2 (\ln Z - c) / M} \right)^M
= \Gamma(1 - 2/M)^M e^{- 2 c} Z^2
.
\end{align*}
Therefore the squared bias, variance and MSE of the estimator $\hat{Z}$ are, respectively:
\begin{align*}
\bias(\hat{Z})^2 &=
(\IE[\hat{Z}] - Z)^2
= Z^2 \left( \Gamma(1 - 1/M)^M e^{- c} - 1 \right)
,
\\
\var(\hat{Z})
&= \IE[\hat{Z}^2] - \IE[\hat{Z}]^2
= Z^2 \left( \Gamma(1 - 2/M)^M e^{- 2 c} - \Gamma(1 - 1/M)^{2M} e^{-2 c} \right)
,
\\
\MSE(\hat{Z}) &=
\bias(\hat{Z})^2 + \var(\hat{Z})
= Z^2 \left( \Gamma(1 - 2/M)^M e^{- 2 c} - 2 \Gamma(1 - 1/M)^M e^{- c}  + 1 \right)
.
\end{align*}
These formulas hold for $M > 2$ where the moment generating functions are defined. For $M = 1$ the estimator has infinite bias (and infinite variance), and for $M = 2$ it has infinite variance. Figure~\ref{fig:estimators_Ms} (left) shows the functional dependence of $\MSE(\hat{Z})$ on the number of samples $M \geq 3$, in units of $Z^2$.

\paragraph{Exponential trick} The Exponential trick yields an unbiased estimator of $1/Z$, and we can turn it into a consistent estimator of $Z$ by inverting it:
\begin{equation*}
\hat{Z}
:= \left( \frac{1}{M} \sum_{m = 1}^M X_m \right)^{-1}
\hspace{2em}\text{where}\hspace{2em}
X_1, \ldots, X_M \stackrel{\text{iid}}{\sim} \Exp(Z)
.
\end{equation*}
As $X_1, \ldots, X_M$ are independent and exponentially distributed with identical rates $Z$, their sum follows the Gamma distribution with shape $M$ and rate $Z$. Therefore the estimator $\hat{Z}$ can be written as $\hat{Z} = M Y$, where $Y \sim \InvGamma(M, Z)$. Recalling the mean and variance of the Inverse-Gamma distribution, we obtain:
\begin{align*}
\bias(\hat{Z})^2 &=
(\IE[\hat{Z}] - Z)^2
= Z^2 \left( \frac{M}{M - 1} - 1 \right)
= Z^2 \frac{1}{M - 1}
,
\\
\var(\hat{Z})
&= Z^2 M^2 \frac{1}{(M-1)^2 (M-2)}
,
\\
\MSE(\hat{Z}) &=
\bias(\hat{Z})^2 + \var(\hat{Z})
= Z^2 \frac{M - 2 + M^2}{(M-1)^2 (M-2)}
= Z^2 \frac{M + 2}{(M-1) (M-2)}
.
\end{align*}
Again these formulas hold for $M > 2$ where the relevant expectations are defined: for $M = 1$ the estimator has infinite bias, and for $M \in \{1, 2\}$ it has infinite variance. Figure~\ref{fig:estimators_Ms} (left) shows the functional dependence of $\MSE(\hat{Z})$ on the number of samples $M \geq 3$, in units of $Z^2$. By inspecting the curves we observe that the Gumbel trick estimator requires roughly 45\% more samples to yield the same MSE as the Exponential trick estimator.

\subsection{Estimating $\ln Z$}

A similar analysis can be performed for estimating $\ln Z$ rather than $Z$. In that case the Gumbel trick estimator of $\ln Z$ is unbiased and has variance (and thus MSE) equal to $\frac{1}{M} \frac{\pi^2}{6}$. On the other hand, the Exponential trick estimator is
\begin{equation*}
\widehat{\ln Z}
= -\ln\left( \frac{1}{M} \sum_{m = 1}^M X_m \right)
\hspace{2em}\text{where}\hspace{2em}
X_1, \ldots, X_M \stackrel{\text{iid}}{\sim} \Exp(Z)
.
\end{equation*}
Again $\sum_{m = 1}^M X_m \sim \operatorname{Gamma}(M, Z)$ and by reference to properties of the Gamma distribution,
\begin{align*}
\bias(\widehat{\ln Z})^2 &=
(\IE[\hat{Z}] - Z)^2
= \left( \ln M - (\psi(M) - \ln Z) - \ln Z \right)^2
= \left(\ln M - \psi(M) \right)^2
,
\\
\var(\widehat{\ln Z})
&= \psi_1(M)
,
\\
\MSE(\widehat{\ln Z}) &=
\bias(\widehat{\ln Z})^2 + \var(\widehat{\ln Z})
= \left(\ln M - \psi(M) \right)^2 + \psi_1(M)
,
\end{align*}
where $\psi(\cdot)$ is the digamma function and $\psi_1(\cdot)$ is the trigamma function. Note that the estimator can be debiased by subtracting its bias $(\ln M - \psi(M))$. Figure~\ref{fig:estimators_Ms} (right) compares the MSE of the Gumbel and Exponential trick estimators of $\ln Z$. We observe that the Gumbel trick estimator performs better only for $M = 1$, and even in that case the Exponential trick estimator is better when debiased.

\section{Sum-unary perturbations}
\label{app:sec:UnarySum}

Recall that \emph{sum-unary perturbations} refer to the setting where each variable's unary potentials are perturbed with Gumbel noise, and the perturbed potential of a configuration sums the perturbations from all variables (see Definition~\ref{def:UnaryPerturbationMAP} in the main text). Using sum-unary perturbations we can derive a family $\calU(\alpha)$ of upper bounds on the log partition function (Proposition~\ref{prop:LowRankUpperBoundOnLnZ}) and construct sequential samplers for the Gibbs distribution (Algorithm~\ref{alg:SequentialSampler}). Here we provide proofs for the related results stated in Sections~\ref{sec:LowRank:UpperBounds} and~\ref{sec:LowRank:Clamping}.

\paragraph{Notation} We will write $\pow_{\beta} x$ for $x^{\beta}$, where $x, \beta \in \IR$, when we find this increases clarity of our exposition.

\begin{lemma}[Weibull and Fr\'echet tricks]
\label{lem:WeibullFrechetMagic}
For any finite set $\mathcal{Y}$ and any function $h$, we have
\begin{align*}
	\pow_{-\alpha} \; \sum_{y \in \mathcal{Y}} \; \pow_{-1/\alpha} h(y)
	&=
	\IE_W\left[ \min_{y} \left\{ h(y) \frac{W(y)}{\Gamma(1 + \alpha)} \right\} \right]
	\hspace{1em}\text{where }
	\{ W(y) \}_{y \in \mathcal{Y}} \stackrel{\text{i.i.d.}}{\sim} \Weibull(1, \alpha^{-1})
	\hspace{1.4em}\text{for }
	\alpha \in (0, \infty)
,
\\
	\pow_{-\alpha} \; \sum_{y \in \mathcal{Y}} \; \pow_{-1/\alpha} h(y)
	&=
	\IE_F\left[ \max_{y} \left\{ h(y) \frac{F(y)}{\Gamma(1 + \alpha)} \right\} \right]
	\hspace{1em}\text{where }
	\{ F(y) \}_{y \in \mathcal{Y}} \stackrel{\text{i.i.d.}}{\sim} \Frechet(1, -\alpha^{-1})
	\hspace{1em}\text{for }
	\alpha \in (-1, 0)
.
\end{align*}
\begin{proof}
This follows from setting up competing exponential clocks with rates $\lambda_y = h(y)^{-1/\alpha}$ and then applying the function $g(x) = x^{\alpha}$ as in Example~\ref{ex:WeibullTricks} for the case of the Weibull trick. The case of the Fr\'echet trick is similar, except that $g$ is strictly decreasing for $\alpha \in (-1, 0)$, hence the maximization in place of the minimization.
\end{proof}
\end{lemma}

\subsection{Upper bounds on the partition function}
\label{app:sec:sum_unary:upper}

\begin{hprop}[\ref{prop:LowRankUpperBoundOnLnZ}.]
	For any $\alpha \in (-1, 0) \cup (0, \infty)$, the upper bound $\ln Z \leq \calU(\alpha)$ holds with
	\begin{equation*}
	\calU(\alpha) := n \frac{\ln \Gamma(1+\alpha)}{\alpha} + nc - \frac{1}{\alpha} \ln \IE_{\gamma}\left[ e^{-\alpha U} \right].
	\end{equation*}%
	\begin{proof} We show the result for $\alpha \in (0, \infty)$ using the Weibull trick; the case of $\alpha \in (-1, 0)$ can be proved similarly using the Fr\'echet trick. The idea is to prove by induction on $n$ that $Z^{-\alpha} \geq e^{-\alpha \calU(\alpha)}$, so that the claimed result follows by applying the monotonically decreasing function $x \mapsto - \ln(x) / \alpha$.

	The base case $n = 1$ is the Clamping Lemma~\ref{lem:ClampingLemma} below with $j = n = 1$. Now assume the claim for $n-1 \geq 1$ and for $x_n \in \calX_n$ define
\begin{equation*}
\calU_{n-1}(\alpha, x_1)
:= (n-1) \frac{\ln \Gamma(1+\alpha)}{\alpha} + (n-1)c - \frac{1}{\alpha} \ln \IE_{\gamma}\left[ \exp\left( -\alpha \max_{x_2, \ldots, x_n} \left\{ \phi(x) + \sum_{i = 2}^{n} \gamma_i(x_i) \right\} \right) \right].
\end{equation*}
With this definition, the Clamping Lemma with $j = 1$ states that $\sum_{x_1} \pow_{-1/\alpha} e^{-\alpha \calU_{n-1}(\alpha, x_1)} \leq \pow_{-1/\alpha} e^{-\alpha \calU(\alpha)}$, so:
\begin{align*}
Z^{-\alpha} &\geq
  \pow_{-\alpha} \sum_{x_1 \in \calX_1} \pow_{-1/\alpha} e^{-\alpha \calU_{n-1}(\alpha, x_1)}
  & \text{[inductive hypothesis]}
\\ &\geq
  \pow_{-\alpha} \pow_{-1/\alpha} e^{-\alpha \calU(\alpha)}
  & \text{[Clamping Lemma]}
\\ &=
  e^{-\alpha \calU(\alpha)},
\end{align*}
as required to complete the inductive step.
\end{proof}
\end{hprop}

\begin{hprop}[\ref{prop:LowRank:Ulimit0}.]
The limit of $\calU(\alpha)$ as $\alpha \to 0$ exists and equals $\calU(0) := \IE[U]$, i.e. the Gumbel trick upper bound.
\begin{proof} Recall that $\calU(\alpha) = n \frac{\ln \Gamma(1+\alpha)}{\alpha} + nc - \frac{1}{\alpha} \ln \IE\left[ e^{-\alpha U} \right]$. The first term tends to $n \psi(1) = -cn$ as $\alpha \to 0$ by L'H\^{o}pital's rule, where $\psi$ is the digamma function. The second term is constant in $\alpha$. In the last term, $\IE\left[ e^{-\alpha U} \right]$ is the moment generating function of $U$ evaluated at $-\alpha$, and as such its derivative at $\alpha = 0$ exists and equals the negative of the mean of $U$. Hence by L'H\^{o}pital's rule,
\begin{align*}
- \lim_{\alpha \to 0} \frac{1}{\alpha} \ln \IE\left[ e^{-\alpha U} \right]
= - \lim_{\alpha \to 0} \frac{-\IE[U]}{\IE\left[ e^{-\alpha U} \right]}
= \IE[U]
= \calU(0).
\end{align*}
The claimed result then follows by the Algebra of Limits, as the contributions of the first two terms cancel.
\end{proof}
\end{hprop}

\begin{hprop}[\ref{prop:UpperBoundDerivativeAt0}.]
	The function $\calU(\alpha)$ is differentiable at $\alpha = 0$ and the derivative equals
	\begin{equation*}
	\frac{\d}{\d \alpha} \calU(\alpha) \bigg\rvert_{\alpha = 0}
	= n \frac{\pi^2}{12} - \frac{\var(U)}{2}
	.
	\end{equation*}
	\begin{proof}
	First we show that $\calU(\alpha)$ is differentiable on $(-1, 0) \cup (0,
	\infty)$, and that the limit of the derivative as $\alpha \to 0$ exists and equals $n \pi^2 / 12 - \var(U) / 2$.

	The first term of $\calU(\alpha)$ is $n \frac{\ln \Gamma(1+\alpha)}{\alpha}$, which is differentiable for $\alpha \in (-1, 0) \cup (0, \infty)$ by the Quotient Rule, and its derivative equals
	\begin{equation*}
	\frac{\d}{\d \alpha} n \frac{\ln \Gamma(1+\alpha)}{\alpha}
	= n \frac{\psi(1 + \alpha) \alpha - \ln \Gamma(1+\alpha)}{\alpha^2},
	\end{equation*}
	where $\psi$ is the digamma function (logarithmic derivative of the gamma function). Applying L'H\^{o}pital's rule we note that
	\begin{align*}
	\lim_{\alpha \to 0} \frac{\d}{\d \alpha} n \frac{\ln \Gamma(1+\alpha)}{\alpha} &=
	n \lim_{\alpha \to 0} \frac{\psi(1 + \alpha) + \alpha \psi^{(1)}(1 + \alpha) - \psi(1 + \alpha)}{2 \alpha}
	=
	n \frac{\psi^{(1)}(1)}{2}
	= n \frac{\zeta(2)}{2}
	= n \frac{\pi^2}{12},
	\end{align*}
	where $\psi^{(1)}$ is the trigamma function (derivative of the digamma function), whose value at $1$ is known to be $\zeta(2) = \pi^2 / 6$, the Riemann zeta function evaluated at $2$.

	The second term of $\calU(\alpha)$ is constant in $\alpha$. The last term can be written as $K(-\alpha) / (-\alpha)$, where $K$ is the cumulant generating function (logarithm of the moment generating function) of the random variable $U$. The cumulant generating function is differentiable, and by the Quotient rule
	\begin{equation*}
	\frac{\d}{\d \alpha} \frac{K(-\alpha)}{-\alpha}
	= - \frac{\alpha K'(-\alpha) - K(-\alpha)}{\alpha^2}.
	\end{equation*}
	Applying L'H\^{o}pital's rule we note that
	\begin{align*}
	\lim_{\alpha \to 0} \frac{\d}{\d \alpha} \frac{K(-\alpha)}{-\alpha} &=
	\lim_{\alpha \to 0} \frac{K'(-\alpha) + \alpha K''(-\alpha) - K'(-\alpha)}{2 \alpha}
	=
	\frac{K''(0)}{2}
	= \frac{\var(U)}{2},
	\end{align*}
	where we have used that the second derivative of the cumulant generating function is the variance.

	As $\calU(\alpha)$ is continuous at $0$ by construction, the above implies that it has left and right derivatives at $0$. As the values of these derivatives coincide, the function is differentiable at $0$ and the derivative has the stated value.
	\end{proof}
\end{hprop}

Recall that for a variable index $j \in \{1, \ldots, n\}$ we also defined \emph{partial sum-unary perturbations}
\begin{equation*}
	U_j(x_1, \ldots, x_{j - 1}) := \max_{x_j, \ldots, x_n} \left\{ \phi(\bx) + \sum_{i = j}^n \gamma_i(x_i) \right\},
\end{equation*}
which fix the variables $x_1, \ldots, x_{j-1}$ and perturb the remaining ones.

\begin{hlem}[\ref{lem:ClampingLemma} {\normalfont (Clamping Lemma)}.]
For any $j \in \{ 1, \ldots, n \}$ and any fixed partial variable assignment $(x_1, \ldots, x_{j - 1}) \in \calX_1 \times \cdots \times \calX_{j - 1}$, the following inequality holds with any trick parameter $\alpha \in (-1, 0) \cup (0, \infty)$:
\begin{align*}
	&
	\sum_{x_j \in \calX_j} \IE_{\gamma}\left[ e^{-(n-j) \ln \Gamma(1 + \alpha) - \alpha (n-j) c)} e^{-\alpha U_{j + 1}(x_1, \ldots, x_j)} \right]^{-1/\alpha}
	\\ & \leq
	\IE_{\gamma}\left[ e^{-(n-(j-1)) \ln \Gamma(1 + \alpha) - \alpha (n-(j-1)) c)} e^{-\alpha U_j(x_1, \ldots, x_{j-1})} \right]^{-1/\alpha}
.
\end{align*}
\begin{proof}
	For $\alpha > 0$, from the Weibull trick (Lemma~\ref{lem:WeibullFrechetMagic}), using independence of the perturbations and Jensen's inequality,
	\begin{align*}
	& \pow_{-\alpha} \sum_{x_j \in \calX_j} \pow_{-1/\alpha} \IE_W\left[ \min_{x_{j+1}, \ldots, x_n} \tilde{p}(\bx)^{-\alpha} \prod_{i = j+1}^n \frac{W(x_i)}{\Gamma(1 + \alpha)} \right]
	\\ &=
	\IE_W\left[ \min_{x_j \in \calX_j} \left\{ \IE_W\left[ \min_{x_{j+1}, \ldots, x_n} \tilde{p}(\bx)^{-\alpha} \prod_{i = j+1}^n \frac{W(x_i)}{\Gamma(1 + \alpha)} \right] \frac{W(x_j)}{\Gamma(1 + \alpha)} \right\} \right]
	\\ & \leq
	\IE_W\left[ \min_{x_j, \ldots, x_n} \tilde{p}(\bx)^{-\alpha} \prod_{i = j}^n \frac{W(x_i)}{\Gamma(1 + \alpha)} \right]
	\end{align*}
	Representing the Weibull random variables in terms of Gumbel random variables using the transformation $W = e^{-(\gamma+c)\alpha}$, where $\gamma \sim \Gumbel(-c)$, and manipulating the obtained expressions yields the claimed result.
\end{proof}
\end{hlem}

\newpage

\subsection{Sequential samplers for the Gibbs distribution}
\label{app:sec:sum_unary:sampler}

The family of sequential samplers for the Gibbs distribution presented in the main text as Algorithm~\ref{alg:SequentialSampler} has the same overall structure as the sequential sampler derived by~\citet{hazan_sampling_2013} from the Gumbel trick upper bound $\calU(0)$, and hence correctness can be argued similarly. Conditioned on accepting the sample, the probability that $\bx = (x_1, \ldots, x_n)$ is returned is
\begin{align*}
\prod_{i = 1}^n p_i(x_i) &=
  \prod_{i = 1}^n \frac{e^{-c}}{\Gamma(1 + \alpha)^{1/\alpha}} \frac{\IE_{\gamma}\left[ e^{-\alpha U_{i + 1}(x_1, \ldots, x_i)} \right]^{-1/\alpha}}{\IE_{\gamma}\left[ e^{-\alpha U_i(x_1, \ldots, x_{i-1})} \right]^{-1/\alpha}}
= \frac{e^{-nc}}{\Gamma(1 + \alpha)^{n/\alpha}}  \frac{\left( e^{-\alpha \phi(x_1, \ldots, x_n) }\right)^{-1/\alpha}}{\IE[e^{-\alpha U}]^{-1/\alpha}}
\propto p(x),
\end{align*}
as required to show that the produced samples follow the Gibbs distribution $p$. Note, however, that in practice one introduces an approximation by replacing expectations with sample averages.

\subsection{Relationship between errors of sum-unary Gumbel perturbations}
\label{app:sec:sum_unary:errors}

We write $\bx^{*}$ for the (random) MAP configuration after sum-unary perturbation of the potential function, i.e.,
\begin{equation*}
\bx^{*} := \argmax_{\bx \in \calX} \left\{ \phi(\bx) + \sum_{i = 1}^n \gamma_i(x_i) \right\}.
\end{equation*}
Let $\qsum(\bx) := \IP[\bx = \bx^{*}]$ be the probability mass function of $\bx^{*}$.

The following results links together the errors acquired when using summed unary perturbations to upper bound the log partition function $\ln Z \leq \calU(0)$ using the Gumbel trick upper bound by~\citet{hazan_partition_2012}, to approximately sample from the Gibbs distribution by using $\qsum$ instead, and to upper bound the entropy of the approximate distribution $\qsum$ using the bound due to~\citet{maji_active_2014}.

\begin{hprop}[\ref{prop:UnarySumGapsLink}.]
	Writing $p$ for the Gibbs distribution, we have
	\begin{align*}
	\underbrace{(\calU(0) - \ln Z)}_{\text{error in } \ln Z \text{ bound}}
	+
	\underbrace{\KL{\qsum}{p}}_{\text{sampling error}}
	=
	\underbrace{\IE_{\gamma_i}\left[ \gamma_i(\bx^{*}_i) \right] - H(\qsum)}_{\text{error in entropy estimation}}.
	\end{align*}
\begin{proof}
	By conditioning on the maximizing configuration $\bx^{*}$, we can rewrite the Gumbel trick upper bound $\calU(0)$ as follows:
	\begin{align*}
	\calU(0) &=
	\IE_{\gamma}\left[ \max_{\bx \in \calX}\left\{ \theta(\bx) + \sum_{i = 1}^n \gamma_i(x_i) \right\} \right]
	\\ &=
	\sum_{\bx \in \calX} \qsum(\bx) \left( \theta(\bx) + \IE_{\gamma}\left[ \sum_{i = 1}^n \gamma_i(x_i) \mid \bx = \bx^{*} \right] \right)
	\\ &=
	\sum_{\bx \in \calX} \qsum(\bx) \theta(\bx) + \sum_{i = 1}^n \IE_{\gamma_i}\left[ \gamma_i(x^{*}_i) \right]
	.
	\end{align*}
	At the same time, the KL divergence between $\qsum$ and the Gibbs distribution $p$ generally expands as
	\begin{align*}
	\KL{\qsum}{p}
	&=
	- H(\qsum) - \sum_{\bx \in \calX} \qsum(\bx) \ln \frac{\exp \left( \theta(\bx) \right)}{\sum_{\tilde{\bx} \in \calX} \exp \left( \theta(\tilde{\bx}) \right)}
	\\&=
 - H(\qsum) -  \sum_{\bx \in \calX} \qsum(\bx) \theta(\bx) + \ln Z.
	\end{align*}
	Adding the two equations together and rearranging yields the claimed result.
\end{proof}
\end{hprop}

\newpage

\section{Averaged unary perturbations}
\label{app:sec:UnaryAvg}

\subsection{Lower bounds on the partition function}
\label{app:sec:sum_unary:lower}

In the main text we stated the following two lower bounds on the log partition function $\ln Z$.

\begin{hprop}[\ref{prop:LowRankLowerBoundAlphaSubset}.]
	Let $\alpha \in (-1, 0) \cup (0, \infty)$. For any subset $S \subseteq \{ 1, \ldots, n \}$ of the variables $x_1, \ldots, x_n$ we have $\ln Z \geq$
	\begin{equation*}
	c + \frac{\ln \Gamma(1 + \alpha)}{\alpha} - \frac{1}{\alpha} \ln \IE\left[ e^{- \alpha \max_{\bx} \{ \phi(\bx) + \gamma_{S} (\bx_{S}) \} } \right],
	\end{equation*}
	where $\bx_{S} := \{ x_i : i \in S \}$ and $\gamma_{S} (\bx_{S}) \sim \Gumbel(-c)$ independently for each setting of $\bx_{S}$.
	\begin{proof}
		Let $\bar{S} := \{ 1, \ldots, n \} \setminus S$. First we handle the case $\alpha > 0$. We have trivially that
		\begin{align*}
		\operatorname{pow}_{-\alpha} Z
		&=
		\operatorname{pow}_{-\alpha} \sum_{\bx_{S}}  \sum_{\bx_{\bar{S}}} e^{\phi(\bx_{S}, \bx_{\bar{S}})}
		\leq
		\operatorname{pow}_{-\alpha} \sum_{\bx_{S}}  \max_{\bx_{\bar{S}}} e^{\phi(\bx_{S}, \bx_{\bar{S}})}.
		\end{align*}
		The Weibull trick tells us that $\operatorname{pow}_{-\alpha} \sum_y \operatorname{pow}_{-1/\alpha} h(y) = \IE_W[ \min_y \frac{h(y)}{\Gamma(1 + \alpha)} W(y) ]$ where $\{W(y)\}_{y} \stackrel{iid}{\sim} \Weibull(1, \alpha^{-1})$. Applying this to the summation over $\bx_{S}$ on the right-hand side of the above inequality, we obtain
		\begin{equation*}
		\operatorname{pow}_{-\alpha} Z
		\leq \IE_{W}\left[ \min_{\bx_{S}} \frac{\operatorname{pow}_{-\alpha} \max_{\bx_{\bar{S}}} e^{\phi(\bx_{S}, \bx_{\bar{S}})}}{\Gamma(1 + \alpha)} W(\bx_{S}) \right].
		\end{equation*}
		Expressing the Weibull random variable $W(\bx_{S})$ as $e^{- \alpha (\gamma_{S}(\bx_S) + c)}$ with $\gamma_{S}(\bx_S) \sim \Gumbel(-c)$, the right-hand side can be simplified as follows:
		\begin{align*}
		\operatorname{pow}_{-\alpha} Z &\leq
		\frac{1}{\Gamma(1 + \alpha)} \IE_{\gamma}\left[ \operatorname{pow}_{-\alpha} \max_{\bx_{S}} \max_{\bx_{\bar{S}}} e^{\phi(\bx_{S}, \bx_{\bar{S}})} e^{\gamma_{S}(\bx_S) + c} \right]
		\\&=
		\frac{e^{-\alpha c}}{\Gamma(1 + \alpha)} \IE_{\gamma}\left[ \exp\left( -\alpha \max_{\bx} \left\{ \phi(\bx) + \gamma_{S}(\bx_S) \right\} \right) \right].
		\end{align*}
		Taking the logarithm and dividing by $- \alpha < 0$ yields the claimed result for positive $\alpha$. For $\alpha \in (-1, 0)$ we proceed similarly, obtaining that
		\begin{align*}
		\operatorname{pow}_{-\alpha} Z
		&\geq
		\operatorname{pow}_{-\alpha} \sum_{\bx_{S}}  \max_{\bx_{\bar{S}}} e^{\phi(\bx_{S}, \bx_{\bar{S}})}
		\\ &=
		\IE_{F}\left[ \min_{\bx_{S}} \frac{\operatorname{pow}_{-\alpha} \max_{\bx_{\bar{S}}} e^{\phi(\bx_{S}, \bx_{\bar{S}})}}{\Gamma(1 + \alpha)} F(\bx_{S}) \right],
		\end{align*}
		where $F(\bx(S)) \sim \Frechet(1, -\alpha^{-1})$. Representing these random variables as $e^{- \alpha (\gamma_{S}(\bx_S) + c)}$ with $\gamma_{S}(\bx_S) \sim \Gumbel(-c)$, simplifying as in the previous case and finally dividing the inequality by $-\alpha > 0$ yields the claimed result for $\alpha \in (-1, 0)$.
	\end{proof}
	\label{app:prop:LowRankLowerBoundAlphaSubset}
\end{hprop}

\begin{hcorl}[\ref{corl:LowRankLowerBound}.]
	For any $\alpha \in (-1, 0) \cup (0, \infty)$, we have the lower bound $\ln Z \geq \calL(\alpha)$, where
	\begin{equation*}
	\calL(\alpha) :=
	c + \frac{\ln \Gamma(1 + \alpha)}{\alpha} - \frac{1}{n \alpha} \ln \IE\left[ \exp\left( - n \alpha L \right) \right],
	\end{equation*}
	\begin{proof}
		Applying Proposition~\ref{prop:LowRankLowerBoundAlphaSubset} $n$ times with all singleton sets $S = \{ i \}$ and averaging the obtained lower bounds yields
		\begin{align*}
		\ln Z &\geq
		c + \frac{\ln \Gamma(1 + \alpha)}{\alpha}
		- \frac{1}{n} \sum_{i = 1}^n \frac{1}{\alpha} \ln \IE\left[ \exp\left( - \alpha \max_{\bx} \{ \phi(\bx) + \gamma_i (x_i) \} \right) \right]
		\\&=
		c + \frac{\ln \Gamma(1 + \alpha)}{\alpha}
		- \frac{1}{n \alpha} \ln \IE\left[ \exp\left( - \sum_{i = 1}^n \alpha \max_{\bx} \{ \phi(\bx) + \gamma_i (x_i) \} \right) \right]
		\\&=
		c + \frac{\ln \Gamma(1 + \alpha)}{\alpha}
		- \frac{1}{n \alpha} \ln \IE\left[ \exp\left( - n \alpha \frac{1}{n} \sum_{i = 1}^n \max_{\bx} \{ \phi(\bx) + \gamma_i (x_i) \} \right) \right],
		\end{align*}
		where the first equality used the fact that the perturbations $\gamma_i(x_i)$ are mutually independent for different indices $i$ to replace the product of expectations with the expectation of the product. The claimed result follows by applying Jensen's inequality to swap the summation and the convex $\max_{\bx}$ function, noting that the inequality works out the right way for both positive and negative $\alpha$.
	\end{proof}
\end{hcorl}

Jensen's inequality can be used to relate the general lower bound $\calL(\alpha)$ to the Gumbel trick lower bound $\calL(0)$, showing that the former cannot be arbitrarily worse than the latter:

\begin{proposition}
	For all $\alpha \in (-1, 0)$, the lower bound $\calL(\alpha)$ on $\ln Z$ satisfies
	\begin{equation*}
	\calL(\alpha) \geq \calL(0) + \frac{\ln \Gamma(1+\alpha)}{\alpha} + c
	\end{equation*}
	\begin{proof}
		Apply Jensen's inequality with the convex function $x \mapsto e^{-n \alpha}$ to the last term in the definition of $\calL(\alpha)$, noting that the inequality works out the stated way for $\alpha < 0$.
	\end{proof}
\end{proposition}
Note that $\frac{\ln \Gamma(1+\alpha)}{\alpha} + c \leq 0$ for $\alpha \in (-1, 0)$ so this result does \emph{not} imply that the Fr\'echet lower bounds are tighter than the Gumbel lower bound $\calL(0)$; it merely says that they cannot be arbitrarily worse than $\calL(0)$.

\subsection{Relationship between errors of averaged-unary Gumbel perturbations}
\label{app:sec:sum_average:errors}

In this section we write $\bx^{*}$ for the (random) MAP configuration after average-unary perturbation of the potential function, i.e.,
\begin{equation*}
\bx^{*} := \argmax_{\bx \in \calX} \left\{ \phi(\bx) + \frac{1}{n} \sum_{i = 1}^n \gamma_i(x_i) \right\}.
\end{equation*}
where $\{ \gamma_i(x_i) \mid x_i \in \calX_i, 1 \leq i \leq n \} \stackrel{\text{i.i.d.}}{\sim} \Gumbel(-c)$. Let $\qavg(\bx) := \IP[\bx = \bx^{*}]$ be the probability mass function of $\bx^{*}$. The Gumbel trick lower bound on the log partition function $\ln Z$ due to~\citet{hazan_sampling_2013} is:
\begin{equation}
\ln Z
\geq \calL(0)
= \calL_{\phi}(0)
:= \IE_{\gamma}\left[ \min_{\bx \in \calX} \left\{ \phi(\bx) + \frac{1}{n} \sum_{i = 1}^n \gamma_i(x_i) \right\} \right].
\label{app:eq:gumbel_lower_bound}
\end{equation}

We show that the gap of this Gumbel trick lower bound on $\ln Z$ upper bounds the KL divergence between the approximate distribution $\qavg$ and the Gibbs distribution $p$. To this end, we first need an entropy bound for $\qavg$ analogous to Theorem 1 of~\citep{maji_active_2014}.

\begin{theorem}
	\label{thm:UnaryAvgEntropyLowerBound}
	The entropy of $\qavg$ can be lower bounded using expected values of max-perturbations as follows:
	\begin{equation*}
	H(\qavg) \geq \frac{1}{n} \sum_{i = 1}^n \IE_{\gamma_i}\left[ \gamma_i(x^{*}_i) \right]
	\end{equation*}
\end{theorem}

\begin{remark}
	Theorem 1 of~\citep{maji_active_2014} and this Theorem~\ref{thm:UnaryAvgEntropyLowerBound} differ in three aspects: (1) the former is an upper bound and the latter is a lower bound, (2) the former sums the expectations while the latter averages them, and (3) the distributions $\qsum$ and $\qavg$ of $\bx^{*}$ in the two theorems are different.
\end{remark}

\begin{proof}
	By the duality relation between negative entropy and the log partition function~\citep{wainwright_graphical_2008}, the entropy $H(\qavg)$ of the unary-avg perturb-max distribution $\qavg$ can be expressed as
	\begin{equation*}
	H(\qavg) =
	\inf_{\varphi} \left\{ \ln Z_{\varphi} - \sum_{\bx \in \calX} \qavg(\bx) \varphi(\bx) \right\},
	\end{equation*}
	where the variable $\varphi$ ranges over all potential functions on $\calX$, and $Z_{\varphi} = \sum_{\bx \in \calX} \exp \varphi(\bx)$. Applying the Gumbel trick lower bound on the log partition function gives
	\begin{equation*}
	H(\qavg) \geq
	\inf_{\varphi} \left\{ \calL_{\varphi}(0) - \sum_{\bx \in \calX} \qavg(\bx) \varphi(\bx) \right\},
	\end{equation*}
	Proposition~\ref{prop:BoundConvex} in Appendix~\ref{app:sec:TechnicalResults} shows that $\calL_{\varphi}(0)$ is a convex function of $\varphi$. The expression $- \sum_{\bx \in \calX} q(\bx) \varphi(\bx)$ is a linear function of $\varphi$, so also convex, and thus as a sum of two convex functions, the quantity $\calL_{\varphi}(0) - \sum_{\bx \in \calX} q(\bx) \varphi(\bx)$ within the infimum is a convex function of $\varphi$. Moreover, Proposition~\ref{prop:LowerBoundPartialDerivatives} in Appendix~\ref{app:sec:TechnicalResults} tells us that the partial derivatives can be computed as
	\begin{equation*}
	\frac{\partial}{\partial \varphi(\bx)} \left( \calL_{\varphi}(0) - \sum_{\bx \in \calX} \qavg(\bx) \varphi(\bx) \right)
	= q_{\varphi}(\bx) - \qavg(\bx)
	\end{equation*}
	where $q_{\varphi}(\bx)$ is the unary-avg perturb-max distribution associated with the potential function $\varphi$. Proposition~\ref{prop:UnaryAvgMaxPerturbDistributionContinuous} in Appendix~\ref{app:sec:TechnicalResults} confirms that these partial derivatives are continuous, so we observe that as a function of $\varphi$, the expression $\calL_{\varphi}(0) - \sum_{\bx \in \calX} \qavg(\bx) \varphi(\bx)$ is a convex function with continuous partial derivatives, so it is a differentiable convex function. This is sufficient to establish that the point $\varphi = \phi$ is a global minimum of this function \citep{wright1999numerical}. Hence
	\begin{align*}
	H(\qavg)
	&\geq
	\inf_{\varphi} \left\{ \calL_{\varphi}(0) - \sum_{\bx \in \calX} \qavg(\bx) \varphi(\bx) \right\}
	\\ &=
	\calL_{\phi}(0) - \sum_{\bx \in \calX} \qavg(\bx) \phi(\bx)
	\\ &=
	\sum_{\bx \in \calX} \qavg(\bx) \IE_{\gamma}\left[ \phi(\bx) + \frac{1}{n} \sum_{i = 1}^n \gamma_i(x_i) \mid \bx = \bx^{*} \right]
	- \sum_{\bx \in \calX} \qavg(\bx) \phi(\bx)
	\\ &=
	\frac{1}{n} \sum_{i = 1}^n \IE_{\gamma_i}\left[ \gamma_i(x^{*}_i) \right]
	\end{align*}
	where we conditioned on the maximizing configuration $\bx^{*}$ when expanding $\calL_{\phi}(0)$.
\end{proof}

\begin{remark}
	This proof proceeded in the same way as the proof of~\citet{maji_active_2014} for the upper bound, except that establishing the minimizing configuration of the infimum is a non-trivial step that is actually required in this case. The second revision of~\citep{hazan_high_2016} computes the derivative of $\calU_{\varphi}(0) - \sum_{\bx \in \calX} \qsum(\bx) \varphi(\bx)$, which is similar to our $\calL_{\varphi}(0) - \sum_{\bx \in \calX} \qavg(\bx) \varphi(\bx)$, by differentiating under the expectation.
\end{remark}

Equipped with Theorem~\ref{thm:UnaryAvgEntropyLowerBound}, we can now show a link between the approximation ``errors" of the averaged-unary perturbation MAP configuration distribution $\qavg$ (to the Gibbs distribution $p$) and estimate $\calL(0)$ (to $\ln Z$).

\begin{hprop}[\ref{prop:UnaryAvgGapsLink}.]
	Let $p$ be the Gibbs distribution on $\calX$. Then
	\begin{equation*}
	\underbrace{\ln Z - \calL(0)}_{\text{error in } \ln Z \text{ bound}}
	\geq
	\underbrace{\KL{\qavg}{p}}_{\text{sampling error}}
	\geq 0
	\end{equation*}
	\begin{remark}
		While we knew from~\citet{hazan_sampling_2013} that $\ln Z - \calL(0) \geq 0$ (i.e. that $\calL(0)$ is a lower bound on $\ln Z$), this is a stronger result showing that the size of the gap is an upper bound on the KL divergence between the average-unary perturbation MAP distribution $\qavg$ and the Gibbs distribution $p$.
	\end{remark}
	\begin{proof}
		The Kullback-Leibler divergence in question expands as
		\begin{align*}
		\KL{\qavg}{p} &=
		  - H(\qavg) - \sum_{\bx \in \calX} \qavg(\bx) \ln \frac{\exp \phi(\bx)}{\sum_{\tilde{\bx} \in \calX} \exp \phi(\tilde{\bx})}
		=
		  - H(\qavg) -  \sum_{\bx \in \calX} \qavg(\bx) \phi(\bx) + \ln Z.
		\end{align*}
		From the proof of Theorem~\ref{thm:UnaryAvgEntropyLowerBound} we know that $H(\qavg) \geq \calL(0) - \sum_{\bx \in \calX} \qavg(\bx) \phi(\bx)$, so
		\begin{align*}
		\KL{\qavg}{p} &\leq
		  - \calL(0) + \sum_{\bx \in \calX} \qavg(\bx) \phi(\bx) - \sum_{\bx \in \calX} \qavg(\bx) \phi(\bx) + \ln Z
		=
		  \ln Z - \calL(0).
		\qedhere
		\end{align*}
	\end{proof}
	\label{app:prop:UnaryAvgGapsLink}
\end{hprop}

\newpage
\section{Technical results}
\label{app:sec:TechnicalResults}

In this section we write $\calL(\phi)$ instead of $\calL_{\phi}(0)$ for the Gumbel trick lower bound on $\ln Z$ associated with the potential function $\phi$, see equation (\ref{app:eq:gumbel_lower_bound}).

\begin{proposition}
	\label{prop:BoundConvex}
	The Gumbel trick lower bound $\calL(\phi)$, viewed as a function of the potentials $\phi$, is convex.
	\begin{proof} Convexity can be proved directly from definition. Let $\phi_1$ and $\phi_2$ be two arbitrary potential functions on a discrete product space $\calX$, and let $\lambda \in [0, 1]$. Then
		\begin{align*}
		& \calL(\lambda \phi_1 + (1 - \lambda) \phi_2)
		\\ &=
		\IE_{\gamma}\left[ \max_{\bx \in \calX}\left\{ \lambda \phi_1(\bx) + (1 - \lambda) \phi_2(\bx) + \frac{1}{n} \sum_{i = 1}^n \gamma_i(x_i) \right\} \right]
		\\ &=
		\IE_{\gamma}\left[ \max_{\bx \in \calX}\left\{ \lambda \left( \phi_1(\bx) + \frac{1}{n} \sum_{i = 1}^n \gamma_i(x_i) \right) + (1 - \lambda) \left( \phi_2(\bx) + \frac{1}{n} \sum_{i = 1}^n \gamma_i(x_i) \right) \right\} \right]
		\\ &\leq
		\IE_{\gamma}\left[ \lambda \max_{\bx \in \calX}\left\{ \phi_1(\bx) + \frac{1}{n} \sum_{i = 1}^n \gamma_i(x_i) \right\} + (1 - \lambda) \max_{\bx \in \calX}\left\{ \phi_2(\bx) + \frac{1}{n} \sum_{i = 1}^n \gamma_i(x_i) \right\} \right]
		\\ &=
		\lambda \calL(\phi_1) + (1 - \lambda) \calL(\phi_2),
		\end{align*}
		where we have used convexity of the $\max$ function to obtain the inequality, and linearity of expectation to arrive at the final equality.
	\end{proof}
\end{proposition}

\begin{remark}
	This convexity proof goes through for other (low-dimensional) perturbations as well, e.g.~it also works for $\calU_{\phi}(0)$.
\end{remark}

\begin{proposition}
	\label{prop:LowerBoundPartialDerivatives}
	The Gumbel trick lower bound $\calL(\phi)$, viewed as a function of the potentials $\phi$, has partial derivatives
	\begin{equation*}
	\frac{\partial}{\partial \phi(\tilde{\bx})} \calL(\phi)
	= q_{\phi}(\tilde{\bx})
	\end{equation*}
	where $q_{\phi}$ is the probability mass function of the average-unary perturbation MAP configuration's distribution associated with the potential function $\phi$.
	\begin{proof}
		Let $\tilde{\bx} \in \calX$, so that $\phi(\tilde{\bx})$ is a general component of $\phi$, and let $e_{\tilde{\bx}}$ be the indicator vector of $\tilde{\bx}$. For any $\delta \in \IR$, the change in the lower bound $\calL$ due to replacing $\phi(\tilde{\bx})$ with $\phi(\tilde{\bx}) + \delta$ is
		\begin{align*}
		\calL(\phi + \delta e_{\tilde{\bx}}) - \calL(\phi) &=
		\IE_{\gamma}\left[ \max_{\bx \in \calX}\left\{ \phi(\bx) + \delta \mathbbm{1}\{\bx = \tilde{\bx}\} + \frac{1}{n} \sum_{i = 1}^n \gamma_i(x_i) \right\} \right]
		- \IE_{\gamma}\left[ \max_{\bx \in \calX}\left\{ \phi(\bx) + \frac{1}{n} \sum_{i = 1}^n \gamma_i(x_i) \right\} \right]
		\\&=
		\IE_{\gamma}\left[ \max_{\bx \in \calX}\left\{ \phi(\bx) + \delta \mathbbm{1}\{\bx = \tilde{\bx}\} + \frac{1}{n} \sum_{i = 1}^n \gamma_i(x_i) \right\} - \max_{\bx \in \calX}\left\{ \phi(\bx) + \frac{1}{n} \sum_{i = 1}^n \gamma_i(x_i) \right\} \right]
		\\&=
		\IE_{\gamma}\left[ \Delta(\phi, \delta, \tilde{\bx}, \gamma) \right]
		\end{align*}
		by linearity of expectation, where we have denoted by $\Delta(\phi, \delta, \tilde{\bx}, \gamma)$ the change in maximum due to replacing the potential $\phi(\tilde{\bx})$ with $\phi(\tilde{\bx}) + \delta$. Let's condition on the argmax before modifying $\phi$:
		\begin{align*}
		\calL(\phi + \delta e_{\tilde{\bx}}) - \calL(\phi) &=
		\IE_{\gamma}\left[ \Delta(\phi, \delta, \tilde{\bx}, \gamma) \right]
		=
		\sum_{\bx \in \calX} q_{\phi}(\bx) \IE_{\gamma}\left[ \Delta(\phi, \delta, \tilde{\bx}, \gamma) \mid \bx \text{ is the original argmax} \right]
		\end{align*}
		Now let's condition on the size of the gap $G$ between the maximum and the runner-up:
		\begin{align*}
		\IE_{\gamma}\left[ \Delta(\phi, \delta, \tilde{\bx}, \gamma) \mid \bx \text{ is the original argmax} \right] &=
		\IP(G \leq |\delta|) \IE_{\gamma}\left[ \Delta(\phi, \delta, \tilde{\bx}, \gamma) \mid \bx \text{ is the original argmax}, G \leq |\delta| \right]
		\\&+
		\IP(G > |\delta|) \IE_{\gamma}\left[ \Delta(\phi, \delta, \tilde{\bx}, \gamma) \mid \bx \text{ is the original argmax}, G > |\delta| \right]
		\end{align*}
		Let's examine all four terms on the right-hand side one by one:
		\begin{enumerate}
			\item $\IP(G \leq |\delta|) \to \IP(G = 0) = 0$ as $\delta \to 0$ by monotonicity of measure.
			\item $\IE_{\gamma}\left[ \Delta(\phi, \delta, \tilde{\bx}, \gamma) \mid \bx \text{ is the original argmax}, G \leq |\delta| \right] \leq \delta$ since $|\Delta(\phi, \delta, \tilde{\bx}, \gamma)| \leq |\delta|$ always holds.
			\item $\IP(G > |\delta|) \to \IP(G \geq 0) = 1$ as $\delta \to 0$ by monotonicity of measure.
			\item $\IE_{\gamma}\left[ \Delta(\phi, \delta, \tilde{\bx}, \gamma) \mid \bx \text{ is the original argmax}, G > |\delta| \right] = \delta \mathbbm{1}\{\bx = \tilde{\bx}\}$ since in this case both maximizations in the definition of $\Delta(\phi, \delta, \tilde{\bx}, \gamma)$ are maximized at $\bx$.
		\end{enumerate}
		Therefore, as $\delta \to 0$,
		\begin{equation*}
		\IE_{\gamma}\left[ \Delta(\phi, \delta, \tilde{\bx}, \gamma) \mid \bx \text{ is the original argmax} \right]
		= o(1) o(\delta) + (1 + o(1)) \delta \mathbbm{1}\{\bx = \tilde{\bx}\}
		\end{equation*}
		Putting things together, we have
		\begin{align*}
		\lim_{\delta \to 0} \frac{\calL(\phi + \delta e_{\tilde{\bx}}) - \calL(\phi)}{\delta} &=
		\sum_{\bx \in \calX} q_{\phi}(\bx) \lim_{\delta \to 0} \frac{1}{\delta} \IE_{\gamma}\left[ \Delta(\phi, \delta, \tilde{\bx}, \gamma) \mid \bx \text{ is the original argmax} \right]
		\\ &=
		\sum_{\bx \in \calX} q_{\phi}(\bx) \mathbbm{1}\{\bx = \tilde{\bx}\}
		\\ &=
		q_{\phi}(\tilde{\bx}),
		\end{align*}
		which proves the stated claim directly from definition of a partial derivative.
	\end{proof}
\end{proposition}

\begin{proposition}
	\label{prop:UnaryAvgMaxPerturbDistributionContinuous}
	The probability mass function $q_{\phi}$ of the average-unary perturbation MAP configuration's distribution associated with a potential function $\phi$ is continuous in $\phi$.
	\begin{proof} For any $\bx^{*} \in \calX$ we have from definition
		\begin{align*}
		q_{\phi}(\bx^{*}) &=
		\IP\left[ \bx^{*} = \argmax_{\bx \in \calX}\left\{ \phi(\bx) + \frac{1}{n} \sum_{i = 1}^n \gamma_i(x_i) \right\} \right]
		\\&=
		\IP\left[ \phi(\bx^{*}) + \frac{1}{n} \sum_{i = 1}^n \gamma_i(x^{*}_i) > \max_{\bx \in \calX \setminus \{ \bx^{*} \}}\left\{ \phi(\bx) + \frac{1}{n} \sum_{i = 1}^n \gamma_i(x_i) \right\} \right]
		\\&=
		\IE\left[ \mathbbm{1}\left\{ \phi(\bx^{*}) + \frac{1}{n} \sum_{i = 1}^n \gamma_i(x^{*}_i) > \max_{\bx \in \calX \setminus \{ \bx^{*} \}}\left\{ \phi(\bx) + \frac{1}{n} \sum_{i = 1}^n \gamma_i(x_i) \right\} \right\} \right]
		\end{align*}
		which is continuous in $\phi$ by continuity of $\max$, of $\mathbbm{1}\left\{\cdot > \cdot\right\}$ (as a function of $\phi$)
		and by the Bounded Convergence Theorem.
	\end{proof}
\end{proposition}

\begin{remark}
	The results above show that the Gumbel trick lower bound $\calL(\phi)$, viewed as a function of the potentials $\phi$, is convex and has continuous partial derivatives.
\end{remark}